\documentclass[journal]{IEEEtran}

\usepackage{times}
\usepackage{epsfig}
\usepackage{graphicx}
\usepackage{amsmath}
\usepackage{amssymb}
\usepackage{amsthm}
\usepackage{multirow}
\usepackage{tabu}
\usepackage{booktabs}
\usepackage{bbm}
\usepackage{float}
\usepackage{chngcntr}
\usepackage{xcolor}
\usepackage{mathtools}

\newcommand{\tifs}[1]{{\color{black} #1}}

\ifCLASSINFOpdf
\else
\fi

\begin{document}
%
\title{Adversarial Attack and Defense on Point Sets}
%
%
%

\author{Jiancheng~Yang,~Qiang~Zhang,~Rongyao~Fang,~Bingbing~Ni,~Jinxian~Liu,~and~Qi~Tian,~\IEEEmembership{Fellow,~IEEE}
\thanks{This work was supported by National Science Foundation of China (U20B200011, 61976137). Authors appreciate the Student Innovation Center of SJTU for providing GPUs.}
\thanks{J. Yang, Q. Zhang, R. Fang, J. Liu are with Shanghai Jiao Tong University, Shanghai, China, and with MoE Key Lab of Articial Intelligence, AI Institute, Shanghai Jiao Tong University, Shanghai, China}
\thanks{B. Ni is with Shanghai Jiao Tong University, Shanghai, China, with MoE Key Lab of Articial Intelligence, AI Institute, Shanghai Jiao Tong University, Shanghai, China, and also with Huawei Hisilicon, Shanghai, China}
\thanks{Q. Tian is with Huawei Noah’s Ark Lab, China.} 
\thanks{J. Yang, Q. Zhang and R. Fang contributed equally to this article.} 
\thanks{Corresponding author: Bingbing Ni}}

\maketitle

\begin{abstract}
Emergence of the utility of 3D point cloud data in safety-critical vision tasks (e.g., ADAS) urges researchers to pay more attention to the robustness of 3D representations and deep networks. To this end, we develop an attack and defense scheme, dedicated to 3D point cloud data, for preventing 3D point clouds from manipulated as well as pursuing noise-tolerable 3D representation. A set of novel 3D point cloud attack operations are proposed via pointwise gradient perturbation and adversarial point attachment / detachment. We then develop a flexible {\em perturbation-measurement} scheme for 3D point cloud data to detect potential attack data or noisy sensing data. \tifs{Notably, the proposed defense methods are even effective to detect the adversarial point clouds generated by a proof-of-concept attack directly targeting the defense.} Transferability of adversarial attacks between several point cloud networks is addressed, and we propose an momentum-enhanced pointwise gradient to improve the attack transferability. We further analyze the transferability from adversarial point clouds to grid CNNs and the inverse. Extensive experimental results on common point cloud benchmarks demonstrate the validity of the proposed 3D attack and defense framework.
\end{abstract}


%
\IEEEpeerreviewmaketitle

\section{Introduction}
%
%
%
%
The popularity of 3D sensors, e.g., LiDAR and RGB-D cameras, raises a number of research concerns with 3D vision. \tifs{Increasingly} accessible data makes data-driven deep learning approaches practical to be used in many fields, such as autopilot \cite{zhou2017voxelnet, qi2017frustum}, robotics \cite{jaremo2018density, deng2018ppfnet} and graphics \cite{yan2016perspective, kato2018neural, wang2018pixel2mesh}. Particularly, point cloud is one of the most natural data structures to represent the 3D geometry. Typical CNN-based approaches \cite{su2015multi, qi2016volumetric} not only require pre-rendering the sparse point clouds into unnecessarily voluminous representations, but also introduce quantization artifacts \cite{qi2016pointnet}. PointNet \cite{qi2016pointnet} and DeepSet \cite{zaheer2017deep} pioneer the direction of learning representations on the raw point clouds. By learning representation from \emph{permutation-invariant} and \emph{size-varying} {\bf point sets}, this idea is shown to be effective and efficient, and has achieved remarkable success \cite{pointnetplusplus, pointcnn, wang2018dynamic}.

\begin{figure}[]
	\centering
	\includegraphics[width=8.5cm]{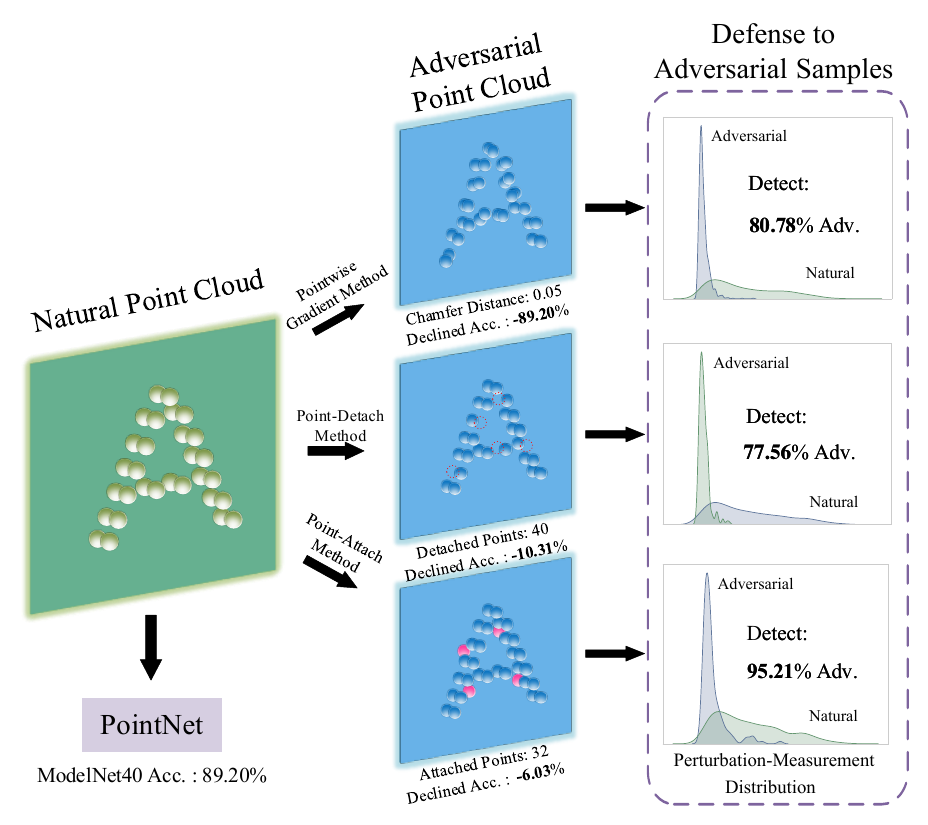}
	\caption{{\bf Illustration of the proposed adversarial attack and defense on point clouds.} A well-tuned PointNet \cite{qi2016pointnet} is vulnerable to our adversarial attacks. Fortunately, the proposed {\em perturbation-measurement} defense strategies, with help of the victim PointNet only, detect most ($>75\%$) of these adversarial \tifs{examples}. Best viewed in color.}
	\label{fig:my_label}
\end{figure}

Considering the popularity and potential use in numerous crucial applications of point clouds, e.g., self-driving, it is urgent to address the security issues. It is proven that even high-performance deep image classifiers are vulnerable to the adversarial attacks \cite{originalAdversarial, tramer2017ensemble, athalye2018obfuscated}, where visually similar \tifs{examples} are generated to mislead the networks to classify incorrectly. We are interested in the security and robustness of the deep nets given adversarial point clouds, especially when learning is based on the raw point clouds.

It is expected that point cloud classification networks are fragile to {\em pointwise gradient} guided optimization adversarial attacks (e.g., FGSM-like methods \cite{FGSM}), due to the generality of these attack methods. We are showing that these pointwise adversarial noises generate visually-similar adversarial \tifs{examples} in 3D space. Besides, considering the realworld 3D sensors, the common issues are sensing noises and \tifs{occluded} points. To this regard, we generate extreme sensing noises by adversarial {\em point-attach} method, and remove possibly vanishing essential points by adversarial {\em point-detach} method. As will be shown later, these prudent attached or detached points are likely to change the decision of well tuned networks. Theoretically, these points are highly related to the \emph{critical points} \cite{qi2016pointnet} existing in the models.

On the other hand, we pursue defense strategies to prevent the proposed adversarial attacks. Based on the observation that the adversarial point clouds destabilize the model output, we explore the intrinsic {\bf robustness of the victim models}. A \emph{perturbation-measurement} defense framework is developed to detect the adversarial \tifs{examples}, which in principle defines a family of easy-to-implement methods. The proposed perturbations are related to {\em input transformation} \cite{guo2017countering} and {\em stochastic activation} \cite{dhillon2018stochastic} methodology, however either the shattered or stochastic gradients ``give a false sense of security" \cite{athalye2018obfuscated}. To this end, instead of utilizing the perturbed \tifs{examples} as defense directly, we measure the statistics of the perturbed outputs of the attacked models to detect the adversarial \tifs{examples}. By changing the threat model with low costs, we detect most of the adversarial \tifs{examples} generated by our attack methods.

Considering the possibility of black-box attacks, the transferability of adversarial point clouds is also analyzed. Several successful point cloud networks are attacked with the proposed methods. Our results imply that high-accuracy models do not enjoy high adversarial robustness. Besides, the adversarial point clouds generated by vanilla gradient-guided attack is poorly transferable, while the momentum-enhanced gradient-guided attack \cite{dong2018boosting} significantly improve the attack transferability. Furthermore, we investigate the transferability between the attack on point clouds and grid CNNs, due to the importance of CNN-based approaches in 3D computer vision.  

The main contributions of this paper are threefold: we 1) address the security issues of adversarial attacks on point clouds, with three attack methods; 2) propose a flexible perturbation-measurement defense framework to detect the adversarial point clouds, with help of the victim models; and 3) analyze the transferability of adversarial point clouds between several point cloud networks, together with that between point clouds and grid CNNs. Our findings could benefit the deep network design for point clouds in terms of robustness and reliability.

\section{Preliminaries}

\subsection{Deep Learning on Point Clouds} \label{sec:dl-on-pc}

Point clouds are widely involved in modern 3D vision research. Shape classification on ModelNet40 dataset \cite{wu20153d} is one of the most important tasks that benchmark this area. 
Voxel-based methods \cite{su2015multi, qi2016volumetric}, which classify pre-rendered point clouds with 2D / 3D CNNs, require computation-intensive rendering and (regular or sparse) convolution operation, hindering the real-time application in practice.

Alternatively, learning directly on point sets provides several advantages on speed and accuracy, with the requirements of 1) \emph{permutation-invariance} and 2) ability to process \emph{size-varying} data. \tifs{The classification of sets is called
multi-instance learning and it was introduced in \cite{dietterich1997solving}. The classifier for deep sets was
first introduced in \cite{pevny2017using} \cite{edwards2016towards}. Recently, PointNets \cite{qi2016pointnet} and DeepSets \cite{zaheer2017deep} achieve great performance, }and it has become a mainstream on reasoning about point clouds. To make the networks able to consume point sets, they use a stack of symmetric functions based on \tifs{max-pooling (or average-pooling)} to build the neural networks. 
PointNet++ \cite{pointnetplusplus} extends PointNet by introducing hierarchical structure with spatial nearest-neighbor graphs, and achieves better classification performance. Subsequent methods (e.g., SO-Net \cite{li2018so}, PointCNN \cite{pointcnn}, PointSIFT\cite{pointsift}) improve the core operations within the hierarchical structure, and DGCNN \cite{wang2018dynamic} introduces a dynamic way to build spatial-neighbor graphs. \tifs{As for theoretically correctness guarantee for these method families, it is provided in \cite{pevny2018approximation}.} We refer these networks as \emph{Point Cloud Networks (PC-Nets)} for simplicity.

\tifs{
For readability, we briefly introduce the core operator $f: \mathbb{R}^{N\times c} \rightarrow \mathbb{R}^{k}$ of PointNets \cite{qi2016pointnet} as follows:

\begin{equation}
	f(X) = \gamma (\max_{x_i \in X} \{h(x_i)\}),
\end{equation}
where $X \in \mathbb{R}^{N\times c}$ denotes a point set, $N$ denotes the number of points, $c$ denotes the dimension of a single point. $k$ is the number of classification categories. $x_i \in \mathbb{R}^{c}$ denotes a single point (a row) in $X$, $\gamma$ and $h$ are two learnable neural networks. }It is proven that this simple function is able to approximate any continuous set function \cite{qi2016pointnet} defined under {\em Hausdorff distance}, 
\begin{equation} \label{eq:haus-dist}
	\mathcal{D}_H(S,S') = \max_{s' \in S'} \min_{s \in S} l_2(s,s').
\end{equation}

Moreover, for a $K$-dimension max-pooling based aggregation, it is also proven that there exists $N_C$ points in the input points ($N_C \leq K$) responsible for the final output, i.e., the output does not change by removing any remaining points. For this reason, the $N_C$ points are called the {\em critical points}. The authors declare robustness on ``small corruptions or extra noise points" \cite{qi2016pointnet} thanks to the property of critical points. On the contrary, we demonstrate that the PointNets are vulnerable to corruptions in adversarial setting; besides, attacks based on the critical points are also successful.

\subsection{Adversarial \tifs{Examples} and Adversarial Attack}

\emph{Adversarial \tifs{example}} was first formalized by Szegedy et al. \cite{originalAdversarial}. Conceptually, an adversarial \tifs{example} $X_a$ is generated by deliberately perturbing a \tifs{benign sample} $X$, which misleads model $f$ to output incorrectly. We focus on untargeted attack in this study; In other words, for some distance metric $\mathcal{D}$ and constrain $\epsilon \in \mathbb{R}$,
\tifs{
\begin{equation}
	\arg \max f(X_a) \neq \arg \max f(X), \quad s.t. \mathcal{D}(X_a,X)<\epsilon.
\end{equation}
}

Generating adversarial \tifs{examples} is called \emph{adversarial attack}. There has been numerous works focusing on attack on image classification, such as fast gradient sign method (FGSM) \cite{FGSM}, least-likely-class iterative method \cite{BIM} and one-pixel attack \cite{onePixel}. Neural networks are proven to be vulnerable to the adversarial attacks, especially in a \emph{white-box} setting, where the adversary has full access, including hidden layers and gradients, to the models.

We explore white-box adversarial attacks on point clouds. The main difference is that images are fixed-dimension {\bf grids}, while point clouds are represented as {\bf sets}. A set is orderless and unstructured, thus distance between sets can be very differently defined from that of grids. In our experiments, asymmetric {\em Chamfer pseudo-distance} ($\mathcal{D}_{C}$) \cite{achlioptas2018learning} is used as a main measurement,
\begin{equation} \label{eq:chamfer-dist}
	\mathcal{D}_{C}(S,S') = \frac{1}{||S'||} \sum_{s' \in S'} \min_{s \in S} l_2(s,s'),
\end{equation}
where $||\cdot||$ denotes points number in a set. Note the Chamfer distance (Eq. \ref{eq:chamfer-dist}) is not so strict as Hausdorff distance (Eq. \ref{eq:haus-dist}): the latter is less tolerant of outliers in the target set $S'$. In practice, outlier noises are very common in 3D sensing, that is the reason why we choose Chamfer distance.

Apart from modifying \tifs{benign samples} point-wisely, it is feasible to perturb point sets by attaching or detaching some points, which is generally impossible for images. However, neither the Hausdorff distance nor the Chamfer distance well measures the number of points changed. For this case, we define a {\em number change measurement},

\begin{equation} \label{eq:number-change}
	\Delta N_{S,S'} = \text{abs}( ||S|| - ||S'|| ).
\end{equation}

\subsection{Defense on Adversarial Attack}

Several prior studies develop defense strategies against adversarial attacks, these works can be regarded as either of two directions: 1) Improve model capacity of the attacked models to classify adversarial \tifs{examples} correctly. For instance, adversarial training is a simple baseline to do this, i.e., expanding training data with adversarial \tifs{examples} \cite{FGSM, Deepfool}; using an additional network to ``denoise" the adversarial \tifs{examples} \cite{DCN} can be also regarded as increasing the model capacity. 2) Detect the adversarial \tifs{examples} and reject them, e.g., LID \cite{LID} suggests a characteristic to detect adversarial \tifs{examples}\footnote{Though the authors emphasize that it ``is not intended as a defense" in personal communication \cite{athalye2018obfuscated}. }. 

In this paper, we mainly focus on attack and defense upon PC-Nets, on the ModelNet40 and MNIST dataset. Only white-box, untargeted attacks are addressed. Our defense strategy follows the ``detect-and-reject" approaches. We note there is a parallel study \cite{xiang2018generating} working on the adversarial attack on point sets, however we develop more attack methods, e.g., point detachment, and a novel defense framework. Besides, we address the transferability of the adversarial point clouds.

\section{Attack Methods}  \label{sec:attack}
\subsection{Principles and Notations}
We represent an input point cloud as $X\in \mathbb{R}^{N\times c}$, where $N$ denotes the number of points, and $c$ denotes the dimension of a single point ($c=3$ for the 3D space). The set of all point clouds in the dataset is defined as $\mathbb{X}$. $f(\cdot)$ is the point cloud network output probability scores of different classes, and $c^{*}(\cdot)$ is the true label. 

Consequently, the attack problem setting is described as follows: by changing input data $X$ to \tifs{$X_a=T(X)$}, where \tifs{$T$} denotes an attack function, given a certain attack budget $\epsilon$ under distance metric \tifs{$\mathcal{D}_{C}$}, the goal is to decrease classification accuracy of attacked model. We define \tifs{$\mathbb{B}$} as the set of adversarial \tifs{examples} under budget $\epsilon$,
\tifs{
\begin{equation}
  \begin{multlined}
	\mathbb{B}: \{X_a \|\ \arg\max f(X)=c^{*}(X) \ \text{and} \\ \arg \max f(X_a)\neq c^{*}(X_a)\}. 
  \end{multlined}
\end{equation}
}
The attack performance is evaluated by post-attack model accuracy,

\begin{equation}
	\eta_\epsilon = \frac{|| \{T(X) \, | \, \arg \max f(T(X))=c^{*}(X)\}||}{||\mathbb{X}||}.
\end{equation}


\newtheorem{proposition}{Proposition}

\subsection{Pointwise Gradient Method} \label{sec:pg}

Inspired by the success of gradient-guided attack methods (e.g., FGSM \cite{FGSM}) on natural images, we first generate the adversarial point clouds by pointwise gradient guided perturbation. \tifs{In this method, the adversarial point clouds are generated by a small perturbation on each point.} Given attack budget $\epsilon$ under the Chamfer distance (Eq. \ref{eq:chamfer-dist}), an adversarial \tifs{example} is obtained by

\begin{equation}
	X_a = T_{PG}(X),\quad s.t.\mathcal{D}_C(X_a, X)<\epsilon,
\end{equation}
where $T_{PG}(x)$ is an iterative perturbation based on pointwise gradient descent.
To decrease the model maximum output in the untargeted-attack setting, we obtain its gradient on the input, via backward pass of output of the ground truth class,

\begin{equation} \label{eq:cal_gradient}
	\nabla X_a = \frac{\partial{f^{(t)}(X_a)}}{\partial{X_a}}, \quad t=c^{*}(X).
\end{equation}



For FGSM, an adversarial \tifs{example} is obtained by:
\begin{equation}
	X_a^{(n+1)} = X_a^{(n)} - \alpha \text{sign}{\nabla X_a^{(n)}},
\end{equation}
where function $sign(x)=\mathbf{1}(x\geq 0)-\mathbf{1}(x<0)$ ensures that the magnitude of the single-step change is maximized with the $\epsilon$-restriction under certain distance $\mathcal{D}$.

In our experiments, we use a different iteration formula based on $l_2$-normalized gradients \cite{L2Norm},
\begin{equation} \label{eq:iteration}
	X_a^{(n+1)} = X_a^{(n)} - \alpha \frac{\nabla X_a^{(n+1)}}{l_2(\nabla X_a^{(n+1)})},
\end{equation}
which leads to a better attack performance and more stable convergence. We call the iterative attack with vanilla gradient descent (Eq. \ref{eq:iteration}) as Pointwise Gradient (PG) Method.

The following proposition declares a theoretical guarantee for universal feasibility of the Pointwise Gradient Method. \tifs{Proof is provided in Appendix \ref{sec:proposition-proof}.}

\begin{proposition}
	Given any point cloud dataset $\mathbb{X}$, $\exists{\epsilon} \in \mathbb{R}$, $\exists{X \in \mathbb{X}}$ s.t. $\arg\max f(X)=c^{*}(X)$, $\exists{X_a}=T_{PG}(X): $ $\mathcal{D}_C(X,X_a)<\epsilon$ and $\arg\max f(X_a)\neq c^{*}(X) $.
\end{proposition}

Apart from vanilla gradient descent, the iteration enhanced with momentum is shown to be more effective on producing transferable adversarial \tifs{examples} \cite{dong2018boosting}. We thereby introduce a Momentum-Enhanced Pointwise Gradient (MPG) Method. A momentum factor $\mu$ is introduced to accumulate the gradient, which is typically set to be $1$ in our experiments,

\begin{equation}
g^{(n+1)} = \mu g^{(n)}+\nabla \frac{X_a^{(n)}}{l_2(\nabla X_a^{(n+1)})}, g^{(0)} =0.
\end{equation}

We then apply the accumulated gradient $g^{(n+1)}$ on the victim sample, 

\begin{equation}
X_a^{(n+1)} = X_a^{(n)} - \alpha g^{(n+1)}.
\end{equation}

In our experiments, the MPG method improve the transferability of adversarial
point clouds over PG.

\subsection{Point-Detach Method}
Pointwise Gradient Method achieves \tifs{a} high success rate, while it adds perturbation to every point, which is hardly possible in practice. We investigate a more realistic scenario by detaching a few points; in physical world, point vanishing is common in 3D sensing, due to occlusion and scale issues. Although prior works \cite{qi2016pointnet,pointnetplusplus,pointcnn} declare robustness on various numbers of input points, we consider this scenario in an adversarial setting by removing some points. Since the Chamfer distance does not measure the points missing from the original set (i.e., the asymmetric Chamfer distance $\mathcal{D}_C \equiv 0$ when detaching points, and the symmetric one is insensitive in that case), we define the attack budget as the number change measurement $\Delta N $ (Eq. \ref{eq:number-change}), thus an adversarial point-detach \tifs{examples} is denoted as 
\begin{equation}
	X_a = T_{PD}(X)\subseteq X,\quad s.t.\Delta N_{X,X_a} \leq N_d,
\end{equation}
where $T_{PD}$ defines a point-detach perturbation.

For PointNets \cite{qi2016pointnet}, we develop a point-detach strategy utilizing the {\em critical point} property (see Section \ref{sec:dl-on-pc}). Recall that the model output changes if and only if the missing point is one of the critical points. 
As the critical points bounded by the $K$-dimension max-pooling layer are \tifs{not relevant with the last several fully-connected and softmax layers whose function is to output the probability of each class}, to efficiently achieve the untargeted attack, we define a {\em class-dependent importance} $\mathcal{I}^{(i)}$ via Taylor first-order approximation of point $i$'s contribution. 

Denote $H_X \in \mathbb{R}^{N \times K}$ as neural networks features with input $X$, before max-pooling aggregation, its gradient matrix $w.r.t.$ the true class output is

\begin{equation}
	\nabla H_X = \frac{\partial}{\partial H_X}f^{(t)}(X),\quad t=c^{*}(X).
\end{equation}
Note that $\nabla H_X \in \mathbb{R}^{N \times K}$ is a sparse matrix with non-zero only at the critical points, i.e.,
\begin{equation}
	\nabla H_X^{(i,j)} \neq 0, \, iff \, i = \mathop{\arg\max}_{k}H_X^{(k,j)} \,for \,j=1,..,K.
\end{equation}

To count the value change $\Delta$ of channel $i$ if its critical point is detached, we introduce a {\em substitute vector} $G \in \mathbb{R}^{K}$, where $G^{(j)}$ is the second largest value the channel $j$ (the $j_{th}$ column of $H$). We then define

\begin{equation}
	\Delta^{(j)} = \max_i H^{(i,j)} - G^{(j)}.
\end{equation}

Thereby, using a first-order Taylor approximation, the class-dependent importance of point $i$ is,
\begin{equation}
	\mathcal{I}^{(i)} = \sum_j^K \nabla H_X^{(i,j)} \cdot \Delta^{(j)}.
\end{equation}

To confuse the attacked network, we apply a greedy strategy, by iteratively detaching the most important point dependent on the true class, until $N_d$ points are detached. In every iteration, the importance order of the remaining points may change. For this reason, we re-compute the class-dependent importance for each remaining point for every iteration, which makes our point-detach method an $O(N\cdot N_d)$ algorithm.

\subsection{Point-Attach Method}
Similarly, by attaching a few points at appropriate positions, we expect another variant of adversarial attack on point clouds,
\begin{equation}
	X_a = T_{PA}(X)\supseteq X.
\end{equation}
where $T_{PA}$ is attaching some points to the original point cloud data with the restriction:
\begin{equation}
	\Delta N_{X,X_a} \leq N_a, \quad \mathcal{D}_C(X_a,X)<\epsilon.
\end{equation}
Define $X_a=X\cup X'_a$, where $X'_a$ is the attached points, thereby $\mathcal{D}_C(X'_a,X)=\mathcal{D}_C(X_a,X)$, and $||X'_a||\leq N_a$.
We initialize $X'_a$ randomly, following Eq. \ref{eq:cal_gradient}, we replace $\nabla X$ by,
\begin{equation}
	\nabla X'_a = \frac{\partial{[f^{(t)}(X_a)+\lambda*\mathcal{D}_c(X,X'_a)]}}{\partial{X'_a}}.
\end{equation}
Where $\lambda*\mathcal{D}_c(X,X'_a)$ is an Lagrange multipier to restrict the attached points to move around surface of point cloud objects. In our experiments $\lambda=0.001$ empirically. We use Eq. \ref{eq:iteration} to only update the attached points without changing the original point cloud. This iteration stops till exceeding the adversarial budgets. 

The proposition below further provides a theoretical guarantee for our Point-Attach Method given point cloud datasets. \tifs{Proof is provided in Appendix \ref{sec:proposition-proof}.}
\begin{proposition}
	Given any point cloud dataset $\mathbb{X}$, $\exists{\epsilon} \in \mathbb{R}$, $\exists{N_a} \in \mathbb{N}$, $\exists{X \in \mathbb{X}}$ s.t. $\arg\max f(X)=c^{*}(X)$, $\exists{X_a}=T_{PA}(X): $ $\mathcal{D}_C(X,X_a)<\epsilon$, $\Delta N_{X,X_a} \leq N_a$, and $\arg\max f(X_a)\neq c^{*}(X) $.
\end{proposition}

\newtheorem{assumption}{Assumption}

\section{Defense Methods}

\subsection{Principle and Notation}
Given a test sample $X\in \mathbb{R}^{N\times c}$, unknown whether it is an adversarial \tifs{example} $X\in \mathbb{B}$ or a \tifs{benign sample} $X\in \mathbb{X}$, our defense methods are supposed to detect $X\in \mathbb{B}$ and reject the adversarial point clouds. 

Observed that the outputs of adversarial point clouds are less stable than the natural ones facing small perturbation, we assume the adversarial \tifs{examples} exist in a narrow and very structured sub-space in a high-dimension input space; in other words, by perturbing inputs non-directionally, we expect the attack to be less aggressive with help of the intrinsic robustness of the attacked models.

To this regard, we propose several perturbation methods, which are related to {\em input transformation} \cite{guo2017countering, xie2017mitigating} and {\em stochastic activation} \cite{dhillon2018stochastic} methodology on natural images. However, we do not use the final output given the perturbation. For small perturbation, the predictive class does not change in most cases; besides, a perturbation-only defense may be mitigated since the ``obfuscated gradients give a false sense of security" \cite{athalye2018obfuscated}. Instead of utilizing the perturbation as a direct defense, we argue that statistics of the outputs provide rich information to detect the adversarial \tifs{examples}.

Our defense framework follows a {\em perturbation-measurement} principle, which applies perturbation methods $\mathcal{P}(\cdot)$ on $X$ multiple times before measuring particular statistics of outputs, and detects the adversarial \tifs{examples} by thresholding the statistics. We define an $M$-times pertubation input set of $X$ as,
\begin{equation}
	X'_m=\{X'_1,X'_2...X'_i...X'_M \,| \,  X'_i=\mathcal{P}_i(X)\}.
\end{equation}
then an $M$-times perturbation output set is defined as,
\begin{equation}
	O'_m=\{O'_1,O'_2...O'_i...O'_M \,| \,  O'_i=f(X'_i)\},
\end{equation}
where $f$ denotes a PC-Net (PointNet \cite{qi2016pointnet} in our experiments). As will be shown later, the distributions of $O'_m$ are very different between \tifs{benign samples} and adversarial \tifs{examples}. We compute certain statistics over $O'_m$, to capture the difference, and report several metrics as \cite{liang2017enhancing} to evaluate the adversarial \tifs{examples} detection performance.

\paragraph{AUROC} 
AUROC is the Area Under the Receiver Operating Characteristic curve \cite{fawcett2006introduction}, a.k.a. AUC, which serves as an evaluation method widely used in binary classification. Note AUROC is threshold-free and insensitive to class imbalance. In our setting, AUROC measures the separability of the adversarial \tifs{examples} and the natural \tifs{examples}. As we use imperfect PC-Nets in the defense, we consider two situations for fair evaluation: 1) $All$: the AUROC between adversarial \tifs{examples} and {\bf all} natural \tifs{examples} in the test set, and 2) $Correct$: the AUROC between adversarial \tifs{examples} and only the {\bf correctly classified} natural \tifs{examples} in the test set.

\paragraph{Defense Detection Rate (DDR) (@ $t\%$)} DDR measures the detection sensitivity at a threshold where the specificity is $1-t\%$, appearing as a single point on the ROC curve. In other words, it measures how many adversarial \tifs{examples} are detected when $t\%$ of natural \tifs{examples} is incorrectly rejected. Note we define the adversarial \tifs{examples} as the positive class.

In the following sections, we instantiate the framework into several perturbations and measurements, to adapt to different attack scenarios.

\subsection{Perturbation Methods} 
\label{defense:perturbation}

\paragraph{Gaussian Noising}
As common practices \cite{vincent2008extracting, DCN} for robustness test in machine learning, we first add Gaussian noises $\rho$ to $X$, named \textit{Gaussian Noising Method} $\mathcal{G}_i^\sigma(\cdot)$.
\begin{equation}
	\mathcal{P}_i(X)=\mathcal{G}_i^\sigma(X) \triangleq X+\rho_i, \quad s.t.\ \rho_i \sim \{N(0 ,{\sigma ^2})\}_{N\times c}.
\end{equation}
The noises are i.i.d. sampled from a Gaussian distribution $N(0 ,{\sigma ^2})$. Adding non-directional Gaussian noise to $X$ helps the attacked models to escape from the narrow adversarial sub-space, which enables an effective follow-up detection.

\paragraph{Quantification}
Motivated by the fact that adversarial perturbations are by definition small in magnitude, we define a \textit{Quantification Method} $\mathcal{Q}_i^\mu(\cdot)$ to convert the inputs into low numerical precision with multiple quantification levels, 
\begin{equation}
	\mathcal{P}_i(X)=\mathcal{Q}_i^\mu(X)\triangleq \lfloor X\times \frac{M}{\mu \times i} \rfloor \times (\frac{\mu \times i}{M}),
\end{equation}
where $\mu$ defines as a max quantification level. For $i=1,...,M$, quantification level ranges from $\mu / M$ to $\mu$. With larger quantification level, $X'$ appears more distorting. PC-Nets are robust to natural \tifs{examples} with a degree of distortion, while we observe that adversarial \tifs{examples} are vulnerable to quantification, resulting in chaotic distributions of outputs from the classifier and distinguishable statistics.

\paragraph{Random Sampling}

The above perturbations are defined on the Euclidean space, we then define a perturbation on the number of points changed $\Delta N$. For $X$ contains $N$ points, we randomly sample $n$ ($n < N$) points $s_i$ from $X$ without replacement, named \textit{Random Sampling Method} $\mathcal{S}_i^n(X)$.
\begin{equation}
	\mathcal{P}_i(X)=\mathcal{S}_i^n(X)=\{\mathbbm{1}_x x \,|\,x \in X,\, \mathbbm{1}_x \sim Ber(0.5)\},
\end{equation}
where $\mathbbm{1}_x$ is sampled from $Bernoulli(0.5)$ distribution to indicate the existence of point $x$ in the post-sampled set. As expected, the sampling perturbation is very effective against the point-attach attacks (Section \ref{sec:defense-performance}). 



\subsection{Measurement Methods}
Given a perturbation output set $O'_m$, for each $O'_i \in O'_m$, 
\begin{equation}
	O'_i =\{o'_{i1},o'_{i2}...o'_{ij}...o'_{iN_c}\},\quad s.t. \ o'_{ij} \in [0,1],
\end{equation}
where $o'_{ij}$ denotes the confidence score of class $j$, and $N_c$ is the number of output classes.

\paragraph{Set-Indiv Variance Measurement}
As observed that the adversarial \tifs{examples} destabilize the model outputs, we first consider to measure the diversity of confidences on the output classes, instead of using an entropy-based measurement, we find $variance$ is more distinguishable in practice, thereby a \textit{Set-Indiv Variance Measurement} $\mathrm{SIV}(\cdot)$ is defined as
\begin{equation}
	\mathrm{SIV}(O'_m)=\frac{1}{N_c}\sum_{k=1}^{N_c}\mathrm{Var}_{i \in {1,2,...M}}(o'_{ik}).
\end{equation}

The variance of each class's confidence set is computed before averaged to $\mathrm{SIV}$ measurement, which is empirically the most effective in most cases (Section \ref{sec:defense-performance}).

\paragraph{Max Confidence-Based Measurement} 
\label{defense:confidence}
Inspired by prior works on out-of-distribution detection \cite{hendrycks2016baseline, liang2017enhancing}, we propose to use the max confidence score to detect adversarial \tifs{examples}. In our {\em perturbation-measurement} framework, we use the max confidence scores statistically. 

We define a \textit{Confidence Average Measurement} $\mathrm{CoA}(\cdot)$,
\begin{equation}
	\mathrm{CoA}(O'_m)=\frac{1}{M}\sum_{i=1}^{M}(\max \limits_{j \in {1,2,...N_c}}(o'_{ij})),
\end{equation}
which measures the average of max confidence scores.

Besides, we also measure the variance of max confidence scores, by designing \textit{Confidence Variance Measurement} $\mathrm{CoV}(\cdot)$,
\begin{equation}
	\mathrm{CoV}(O'_m)=\mathrm{Var}_{i \in {1,2...M}}(\max \limits_{j \in {1,2...N_c}}(o'_{ij})).
\end{equation}

With the following two assumptions, we provide a theoretical guarantee for the effectiveness of Max Confidence-based Measurement (i.e., $CoV$ and $CoA$). \tifs{Proof is provided in Appendix \ref{sec:proposition-proof}.}

\begin{assumption}
	PC-Nets are local continuous convex or concave functions around natural \tifs{examples}. 
\end{assumption}

\begin{assumption} The proportion of adversarial \tifs{examples} in the local area around natural \tifs{examples} is small enough, i.e., given natural \tifs{examples} $X_n$, 
	\begin{align*}
		&\exists{\delta}>0, \text{define} \, \mathbb{D}_{\delta}=\{X \,|\,\mathcal{D}_{C}(X,X_n)<\delta\}, \\&\mathbb{B}_\delta= \{X_a \, is \, adversarial \,\tifs{examples} \,|\, \mathcal{D}_C(X_a,X_n)<\delta\} \subset \mathbb{D}_\delta,
		\\&\exists{\epsilon} \ll 1: \forall{X}, P(X\in\mathbb{B}_\delta)/P(X\in\mathbb{D}_\delta)<\epsilon .
	\end{align*}
\end{assumption}

\begin{proposition}
	Given any sample $X$, it can \textbf{always} be detected whether it is an adversarial \tifs{example} or a \tifs{benign sample} by Max Confidence-Based Measurement  ($CoV$ for Convex functions and $CoA$ for Concave functions), with Gaussian Noising or Quantification perturbation.
\end{proposition}

\tifs{
\subsection{Attack over the Defenses}

Observing the literature on adversarial attack and defense, there is never a defense technique strong enough in a long term \cite{Carlini2016TowardsET,tsipras2018robustness,grosse2018limitations}. To further evaluate the robustness of the proposed perturbation-measurement defense methods, we design a proof-of-concept attack targeting the proposed defense methods. Considering the connection between randomness-based defense \cite{xie2017mitigating} and the proposed defense strategy, we develop a variant of Pointwise-Gradient attack, aimed at evading the proposed defense. We name this attack method as Expectation-over-Transformation Pointwise-Gradient (EoTPG) attack, since it is inspired from Expectation over Transformation (EoT) technique \cite{athalye2018obfuscated,pmlr-v80-athalye18b}. It is also very related to the PGD method \cite{madry2017towards}. Essentially, each iteration of the EoTPG attack uses the average gradient by perturbed samples from the mentioned perturbation. 

Modified by the vanilla Pointwise-Gradient attack, each step of EoTPG attack uses the following equation:

\begin{equation} \label{eq:aod-iteration}
	X_a^{(n+1)} = X_a^{(n)} - \alpha \frac{E(\nabla \mathcal{P}_i(X_a^{(n+1)}))}{l_2(E(\nabla \mathcal{P}_i(X_a^{(n+1)})))},
\end{equation}
\begin{equation}
    E(\nabla \mathcal{P}_i(X_a^{(n+1)}))=\sum_i^M(\frac{\partial{f^{(t)}(\mathcal{P}_i(X_a))}}{\partial{\mathcal{P}_i(X_a})})/M,
\end{equation}
where $\mathcal{P}$ is one of the proposed perturbation (Gaussian Noising, Quantification or Random Sampling) used in the defense methods, and $M$ denotes the times of perturbation in each attack step. $M$ is simply set $100$ in our experiments.

Different from the effectiveness of Expectation-over-Transformation techniques \cite{athalye2018obfuscated} on randomness-based defense methods \cite{xie2017mitigating} on natural images, our proposed defense techniques are surprisingly effective to detect the adversarial point clouds generated by the attack method targeting the defense (Section \ref{sec:aod=results}), which empirically proves the robustness of the propose defense methods on detecting adversarial point clouds.
}

\section{Results}
In this section, we evaluate the attack performances of our methods on a simple yet effective PC-Net, PointNet (with T-Net) \cite{qi2016pointnet}. At the same time, we conduct defense experiments on these adversarial point clouds with positive results. All results are reported on ModelNet40 dataset \cite{wu20153d} of 40-category CAD models if not specified otherwise. We use the official split with 9,843 \tifs{examples} for training, and 2,468 \tifs{examples} for test / attack / defense. 1,024 points are uniformly sampled from the mesh surfaces as in PointNet. 

\subsection{Attack Performance}
In order to verify the performance of our attack methods, extensive experiments are conducted on the ModelNet40 dataset. Figure \ref{fig:add} illustrates several \tifs{examples} attacked by the proposed methods, \tifs{and we display the adversarial point clouds generated by Pointwise-Gradient method under various Chamfer distances in Figure \ref{fig:attack_vs_chamfer}.} As demonstrated, the generated adversarial point clouds are visually indiscernible. For Pointwise-Gradient method and Point-Attach method, unnoticeable perturbation is applied on each attacked point, in case that the layout is destroyed. For Point-Detach method, the detached points are the critical points, which are often located on the edges and corners of the point clouds. In physical environments, these points are more likely to be vanishing for 3D sensing due to occlusion and scale issues.

\begin{figure}[!htb]
	\centering
	\includegraphics[width=\linewidth]{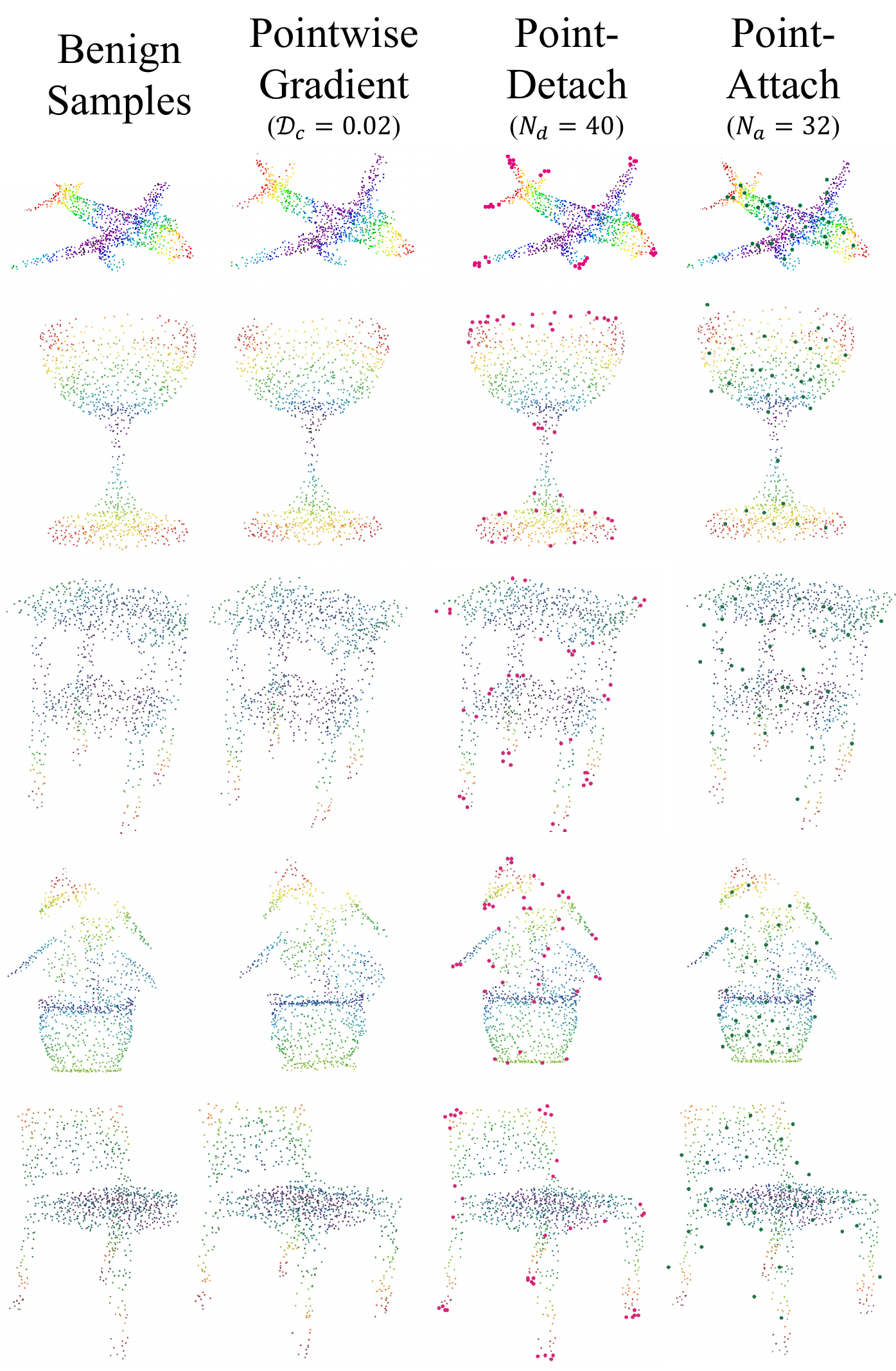}

	\caption{\textbf{Illustration of natural point clouds and adversarial point clouds produced by the proposed attack methods.} Point clouds by Pointwise Gradient method are generated with $\mathcal{D}_c = 0.02$. For Point-Detach method, 40 red points denotes the detached points. For Point-Attach method,  32 attacked points colored green are attached.}
	\label{fig:add}
\end{figure} 

\begin{table}

		\caption{\textbf{Attack performance on ModelNet40 dataset.} Metric ($\eta$) is the accuracy (\%) on classification after attacking.}
	
	\centering
	\begin{tabu}to 0.47\textwidth{X[c]X[c]X[c]X[c]X[c]X[c]X[c]X[c]}
		\toprule[1pt]
		\multicolumn{7}{c}{Pointwise Gradient Method}\\ \hline
		$\mathcal{D}_C$ & 0.01  & 0.02  & 0.03  & 0.04  & 0.05  & 0.06\\
		$\eta$ & 86.58 & 59.21 & 39.38 & 29.25 & 9.25  & 0.00\\\hline
		
		\multicolumn{7}{c}{Point-Detach Method}\\ \hline
		$N_d$ & 8 & 12 & 20 & 40 & 100 & 200 \\  
		$\eta$  & 87.54 & 86.53 & 83.89 & 78.89 & 69.32 & 67.58 \\ \hline
		\multicolumn{7}{c}{Point-Attach Method}\\ \hline
		$N_a$ & 16 & 16 & 32 & 32 & 100 & 100 \\ 
		$\mathcal{D}_c$ & 0.1 & 0.2 & 0.3 & 0.4 & 0.5 & 2.0 \\ 
		$\eta$  & 86.82 & 85.14 & 84.26 & 83.89 & 80.28 & 75.12 \\ 
		\bottomrule[1pt]
	\end{tabu}

	\label{attack_form}
\end{table}

The attack performance is depicted in Table \ref{attack_form}, where the attack parameters correspond to different attack intensities. By applying the proposed Pointwise Gradient method, the accuracy of PointNet could be reduced to $0\%$, which means that all natural \tifs{examples} are successfully attacked. 

The Point-Detach method and Point-Attach method, are less effective than the Pointwise-Gradient attack in terms of accuracy reduction, however these attack are much more physically feasible, and these \tifs{examples} easily deceive human eyes \tifs{(see illustration of Point-Detach and Point-Attach in Figure \ref{fig:add})}. Under looser constraints, the attack performance can be even higher. It is worth noting that the adversarial PD and PA methods are much more effective than random point detachment / attachment. For illustration, we visualize (Figure \ref{fig:pd-attack}) the attack performance under PD attack with various detached points $N_d$. The victim PointNet is robust with random point detachment (with less than 200 detached points), however it is vulnerable in an adversarial setting. 

\begin{figure}[!htb]
	\centering
	\includegraphics[width=0.9\linewidth]{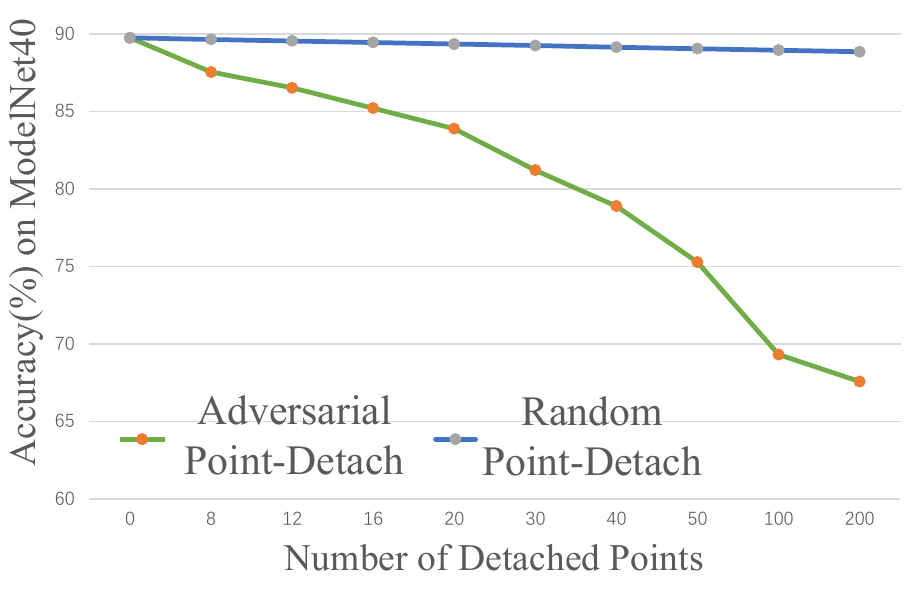}

	\caption{\textbf{Point-Detach attack performance on ModelNet40 versus detached points $N_d$.} Accuracy with random point detached are also visualized for comparison.}
	\label{fig:pd-attack}
\end{figure}

\subsection{Defense Performance}
\label{sec:defense-performance}

\begin{figure}[!htb]
	\centering
	\includegraphics[width=4cm]{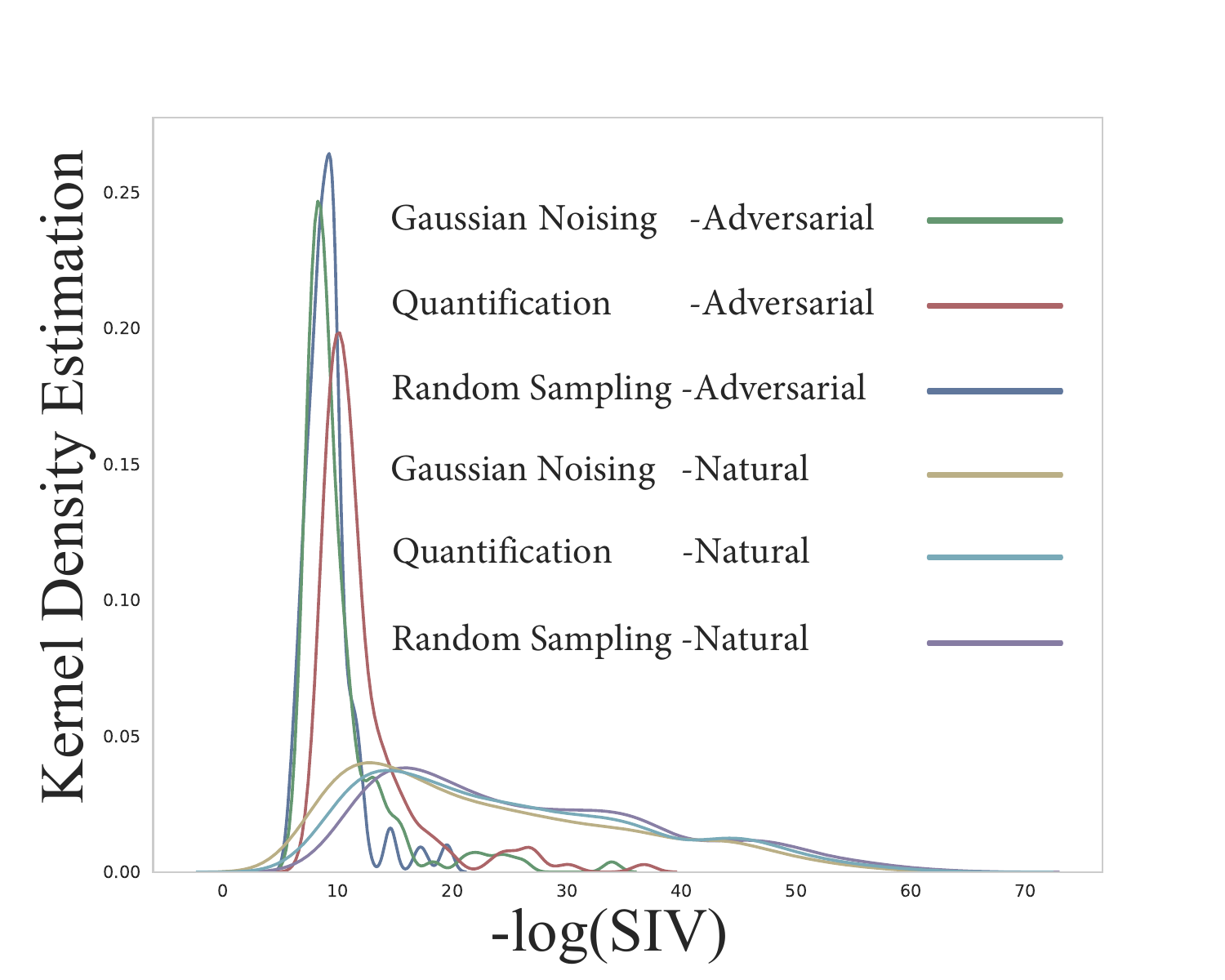}
	\includegraphics[width=4cm]{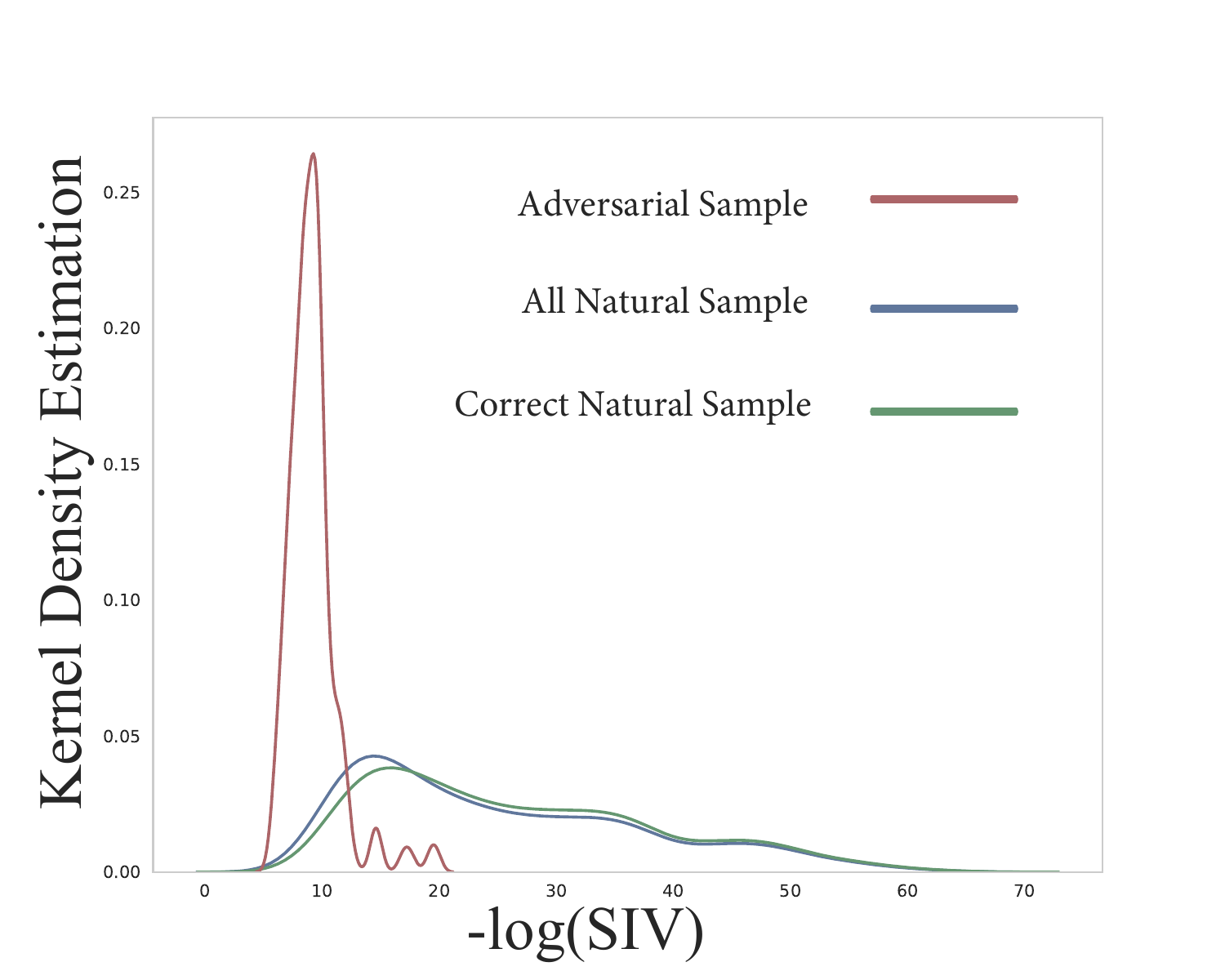}
	\caption{\textbf{Left}: Kernel density estimation (KDE) of $-log(SIV)$ towards 32-point Point-Attach method, which shows different statistics on distribution of adversarial \tifs{examples} and $Correct$ natural \tifs{examples}. \textbf{Right}: KDE with adversarial \tifs{examples}, $All$ natural \tifs{examples}, and $Correct$ natural \tifs{examples}, which reveals the Defense Error II in Section \ref{Discussion}.}
	\label{fig:defense}
\end{figure}

\begin{table*}
		\caption{\textbf{Defense performance with the best perturbation-measurement parameter setting on the proposed attack methods.} The Defense Detection Rates (DDR) are evaluated with $Correct$ natural \tifs{examples}.}
    	\begin{center}
		\begin{tabular}{p{2.7cm}p{1.9cm}p{2cm}p{0.9cm}p{1cm}p{1.6cm}p{1.7cm}p{2.2cm}}
			
			\toprule[1pt]
			\ \ \ Attack Method & \ \ \ \ \ Param. & \ \ \ \ \ Perturb. & Measur. & $\text{AUC}_{All}$ & $\text{AUC}_{Correct}$ & $\text{DDR (@5\%)}$ & $\text{DDR (@10\%)}$\\ \hline
			Pointwise Gradient &  \ \ $\mathcal{D}_C=0.02$ & Quanti. (0.08) & \ \ \  SIV &0.9168 &\ \ \ \ \   0.9381 & \ \ \ \ 61.34 & \ \ \ \ \ 80.78 \\
			\ \ \ \ Point-Detach &  \ \ \ $N_d={20}$ &\ \ GN. (0.012) & \ \ \ SIV &0.9229 &\ \ \ \ \ 0.9517  & \ \ \ \  65.69 & \ \ \ \ \  88.73 \\
			\ \ \ \ Point-Detach &  \ \ \ $N_d={40}$ &\ \ GN. (0.012) & \ \ \ SIV &0.9011 &\ \ \ \ \   0.9343 & \ \ \ \  54.26 & \ \ \ \ \  77.56 \\
			\ \ \ \ Point-Attach &  \ \ \ $N_a={32}$  &\ \ \  RS. (1000) & \ \ \  SIV &0.9729 &\ \ \ \ \   0.9816 & \ \ \ \  93.15 & \ \ \ \ \  95.21\\
			\bottomrule[1pt]
		\end{tabular}

		\label{form:defense}
	\end{center}
\end{table*}

Perturbation-measurement defense strategies are utilized to detect the adversarial \tifs{examples} produced by the proposed attack methods. We illustrate the effectiveness of measurement by the kernel density estimation (KDE) plot of $-log(SIV)$ of natural point clouds and adversarial point clouds, $-log(\cdot)$ is simply for visual purpose. As shown in Figure \ref{fig:defense}, the distributions of the natural \tifs{examples} and adversarial \tifs{examples} are significantly different, which implies that the adversarial \tifs{examples} could be detected with appropriate threshold.

We then quantitatively evaluate the performance of the proposed defense strategies on the adversarial attacks. AUROC and Defense Detection Rate (DDR) are used to measure the defense performance, as depicted in Table \ref{form:defense}, \tifs{where the best defense performance is illustrated with optimal hyperparameters. A complete defense performance for various settings are shown in the Appendix \ref{appendix:defense}.} Note that
two kinds of natural \tifs{examples} subsets are evaluated: 1) $All$ natural \tifs{examples}: the whole test dataset; 2) $Correct$ natural \tifs{examples}: the point clouds correctly classified by victim PC-Nets. Further discussion is given in Section \ref{Discussion} about these two settings.

The proposed defense strategies successfully detect most ($> 75\%$) adversarial \tifs{examples}. 
Given the fact that our defense methods strongly rely on the robustness of the victim PC-Nets, their performance drops with $N_d$ increasing in the defense of Point-Detach attack. When the critical points detached, the fluctuation of outputs decreases, resulting in a decrease of DDR. Though $CoA$ and $CoV$ are theoretically effective on detecting the adversarial \tifs{examples} (Section \ref{defense:confidence}), empirically they perform slightly worse than $SIV$. The complete performance comparison is reported in Appendix \ref{appendix:defense}.


\subsection{Performance of Attack over the Defenses} \label{sec:aod=results}

\begin{table}

	\caption{Comparation of the accuracy ($\eta$) between original vanilla PG attack and EoTPG attack. The attack budge $\mathcal{D}_C$ is $0.05$.}
	
	\centering
	\begin{tabu}to 0.47\textwidth{X[c]X[c]X[c]X[c]X[C]}
		\toprule[1pt]
		 & PG  & EoTPG RS.  & EoTPG Quanti.  & EoTPG GN.  \\ \hline

		$\eta$ & 9.25 & 87.78 & 64.23 & 66.10 \\ 
		\bottomrule[1pt]
	\end{tabu}

	\label{form:eotpg}
\end{table}

\tifs{We applied Expectation-over-Transformation Pointwise- Gradient (EoTPG) attack method with random sampling (RS.), quantification (Quanti.), and Gaussian noising (GN.) perturbations. The EoTPG attacks are applied with $100$-time perturbations for each attack step. Table \ref{form:eotpg} shows the accuracy of PointNet after EoTPG attack method and vanilla PG attack method. The original accuracy of PointNet without any attack is 89.20\%. $\eta$ means the accuracy (\%) of PointNet after attacking, as the previously stated. The attack success rate of EoTPG declines sharply when compared with vanilla PG attack method, which means EoTPG method is harder than original PG method to produce adversarial examples. EoTPG method with random sampling perturbation is the most hardest to produce adversarial examples, whose accuracy only decrease by $1.42\%$. It might result from that the adversarial examples with different number of points have low transferability, which introduces very large variance of the gradients.}

\begin{table}

\caption{Defense performance (AUROC) towards PG and EoTPG attackn.}
	
\begin{center}
\begin{tabu} to 0.47\textwidth{X[c]X[c]X[c]X[c]X[c]X[c]}
\toprule[1pt]
\multicolumn{2}{c}{Defense$\backslash$Attack} & PG & EoTPG RS.       & EoTPG Quanti. & EoTPG GN.                    \\ \hline

\multirow{3}{*}{RS.} & SIV & 0.8469 & 0.9328 & 0.9050 & 0.8486  \\ 
& CoA & 0.7970 & 0.9444 & 0.8133 & 0.8016 \\
& CoV & 0.8429 & 0.9165 & 0.9022 & 0.8465 \\ \hline
\multirow{3}{*}{Quanti.} & SIV & 0.8593 & 0.8732 & 0.8383 & 0.8721  \\
& CoA & 0.8087 & 0.9444 & 0.8133 & 0.8016 \\
& CoV & 0.8493 & 0.9165 & 0.9022 & 0.8465 \\ \hline
\multirow{3}{*}{GN.} & SIV & 0.8267 & 0.9462 & 0.8593 & 0.8432 \\
& CoA & 0.7968 & 0.9444 & 0.8133 & 0.8016 \\
& CoV & 0.8240 & 0.9165 & 0.9022 & 0.8465 \\
			
			\bottomrule[1pt]
		\end{tabu}

		\label{form:defense_eotpg}
	\end{center}
\end{table}

\tifs{Table \ref{form:defense_eotpg} shows the defense AUROC towards vanilla PG attack method and EoTPG attack methods. The four columns refer to vanilla PG attack method, EoTPG with random sampling (of $500$ points), EoTPG with Quantification (of $0.12$), and EoTPG with Gaussian noising (of $0.025$). The EoTPG attacks are applied after $100$ times perturbations for each step. The rows refer to the AUROC of our defense methods with three perturbation methods and three measurement statistics. The parameters of the three defense perturbations are the same as that in EoTPG ($500$ points, $0.12$, and $0.025$ correspondingly). The attack budget $\mathcal{D}_C$ is $0.05$. The defense method is applied with adversarial examples and \textit{Correct} classified natural point clouds. The defense AUROC towards the corresponding EoTPG attack methods are expected to be lower than that towards vanilla PG method. However, in most cases, the defense performances towards corresponding EoTPG attack are better than that that of PG method, which means our defense methods work effectively towards the specifically designed EoTPG attack, even directly targeting the proposed defense strategy.}

\section{Transferability of Adversarial \tifs{Examples}}

In this section, we analyze the transferability of the adversarial \tifs{examples} between various PC-Nets, to address the black-box attack setting. Moreover, the transferability between point clouds and grid CNNs is also investigated, considering the importance of the CNN-based approaches. 

\subsection{Between Various PC-Nets}
Several PC-Nets, e.g., PointNet++ \cite{pointnetplusplus} and Dynamic Graph CNN (DGCNN) \cite{wang2018dynamic} are designed based on PointNet \cite{qi2016pointnet}. PointNet++ improves PointNet by introducing a hierarchical structure based on $k$-nearest neighbor (kNN) graphs, and DGCNN further replace the static kNN graphs by dynamically reconstruct the kNN graphs. \tifs{There are also some previous work discussing the transferability among adversarial examples obtained from different models: \cite{liu2016delving,su2018robustness}. In this study, we explore this issue by evaluating the attack performance and transferability on these PC-Nets.} The evaluated PC-Nets are trained with the official open source code. 

\paragraph{Attack Performance}

\begin{figure}[!htb]
	\centering
	\includegraphics[width=0.9\linewidth]{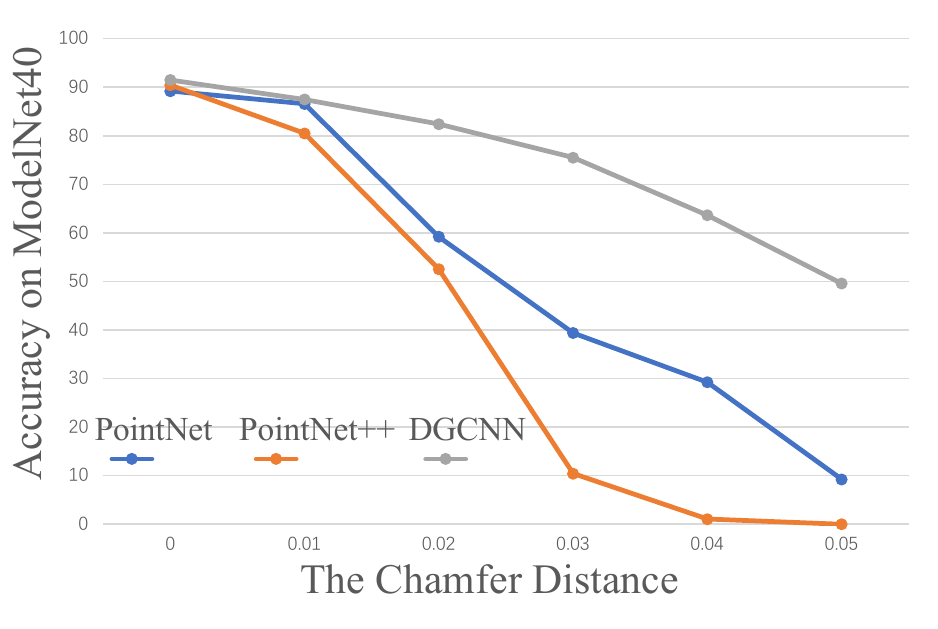}
	\caption{Attacked accuracy on ModelNet40 versus $\mathcal{D}_C$ on various PC-Nets.}
	\label{fig:various-pc-net}
\end{figure} 

We apply the Pointwise Gradient attack on these PC-Nets to evaluate their robustness to the adversarial point clouds. By tuning the attack intensities (the Chamfer distance $\mathcal{D}_C$), the accuracy of PC-Nets after attacking is shown in Figure \ref{fig:various-pc-net}.

Interestingly, PointNet++ performs worse in terms of robustness than PointNet. PointNet++ outperforms PointNet without attack. Please note that the static kNN graphs in the PointNet++ is constructed on the attacked \tifs{examples}. We conjecture that the static kNN graphs of natural \tifs{examples} built in the Euclidean space make the hidden representations lay in a narrow space. The kNN graphs of adversarial point clouds are differently structured, which weaken the robustness of the hidden representations. 

DGCNN is observed more robust than that of PointNet and PointNet++. The DGCNN enjoys highest performance without attacking. Besides, the dynamic edge convolution operation in DGCNN disturbs the gradient backward, which results in ``shattered gradients'' \cite{athalye2018obfuscated} and contributes to high robustness.

\paragraph{Transferability}
To evaluate transferability between these PC-Nets, we transfer adversarial \tifs{examples} generated by one source network with (vanilla) Pointwise Gradient method to \tifs{the ones generated by the other} evaluating PC-Net. As shown in Table \ref{attack_trans}, we evaluate the adversarial \tifs{examples} on the evaluating PC-Net to obtain \textit{post-attacking accuracy}, and the corresponding natural \tifs{examples} to obtain the \textit{pre-attacking accuracy}. 

The evaluation results indicate that poor attack transferability of the adversarial point clouds generated by the Pointwise Gradient attack. The high post-attacking accuracy implies, that only a few adversarial \tifs{examples} are successfully transferred to the evaluating networks. Pointwise Gradient method is not an effective black-box adversarial attack in point cloud cases.

\begin{table}
		\caption{\textbf{Transferability evaluation on Pointwise Gradient method (PG)}. Pre-attacking / post-attacking accuracy are depicted. \tifs{The attack budget $\mathcal{D}_C=0.03$.}}
		
    	\begin{center}
		\begin{tabular}{p{1.4cm}p{1.75cm}p{1.75cm}p{1.75cm}}
			\toprule[1pt]
			From $\backslash$ To& \ \ \ PointNet          & \ PointNet++        & \ \ \ DGCNN             \\ \hline
			PointNet    & $\text{100.0 / 0.00}$       & $\text{94.63 / 94.21}$ & $\text{88.57 / 87.01}$ \\
			PointNet++    & $\text{89.79 / 83.01}$ & $\text{100.0 / 0.00}$       & $\text{84.42 / 81.56}$ \\ 
			DGCNN       & $\text{92.63 / 87.53}$ & $\text{95.89 / 94.63}$ & $\text{100.0 / 0.00}$       \\ \bottomrule[1pt]
		\end{tabular}

		\label{attack_trans}
	\end{center}
\end{table}

We explore further, whether the failure of black-box attack comes from the characteristics of PC-Nets / point clouds, or simply the ineffective attack method. Inspired from the success of momentum-based iterative attack on natural images \cite{dong2018boosting}, we develop a Pointwise Gradient variant enhanced by momentum (see Section \ref{sec:pg}), named Momentum Poinwise Gradient method (MPG). Surprisingly, MPG significantly improve the transferability of the adversarial \tifs{examples} (Table \ref{table:mpg}). Most adversarial \tifs{examples} obtain low post-attacking accuracy on the evaluating networks, though MPG attack performance on the source networks is not observed significantly different from that of PG. It is interesting that adversarial \tifs{examples} generated by DGCNN are less transferable than the another, even though the DGCNN is more robust to the adversarial PG attack.  

\tifs{Our results put a light on the black-box attack based on the transferability of adversarial point clouds. It is interesting to find that even though it seems harder to generate transferable adversarial point clouds with vanilla attack methods, it is still possible to achieve this by more sophisticated attack methods, e.g., the Pointwise-Gradient attack enhanced by the momentum attack \cite{dong2018boosting}. It could be explained by the ``machine learning view'' of adversarial examples \cite{dong2018boosting}: machine learning methods with better generalization could be possibly extended to more transferable black-box attack. Besides, the momentum attack accumulates the historical perturbation in a moving average fashion, which is also very similar to the methods synthesizing robust adversarial examples \cite{pmlr-v80-athalye18b}. Therefore, it is expected that the input diversity method \cite{Xie2018ImprovingTO} and translation-invariant attack \cite{dong2019evading} could also be very promising in generating transferable adversarial point clouds. We leave this for further exploration.}

\begin{table}
	\caption{\textbf{Transferability evaluation on Momentum Pointwise Gradient method (MPG)}. Pre-attacking / post-attacking accuracy are depicted. \tifs{The attack budget $\mathcal{D}_C=0.03$.}}
	
	\begin{center}
		\begin{tabular}{p{1.8cm}p{1.6cm}p{1.75cm}p{1.55cm}}
			\toprule[1pt]
			From $\backslash$ To & \ \ \ PointNet          & \ PointNet++        & \ \ \ DGCNN             \\ \hline
			PointNet    & $\text{100.0 / 0.00}$       & $\text{80.30 / 5.464}$ & $\text{90.07 / 4.801}$ \\
			PointNet++    & $\text{91.35 / 2.275}$ & $\text{100.0 / 0.00}$       & $\text{93.74 / 2.844}$ \\ 
			DGCNN       & $\text{80.81 / 9.091}$ & $\text{71.21 / 19.70}$ & $\text{100.0 / 0.00}$       \\ \bottomrule[1pt]
		\end{tabular}
		\label{table:mpg}  
	\end{center}
\end{table}

\subsection{Between PC-Nets and CNNs}
To fully investigate the transferability of adversarial \tifs{examples} generated by our methods, we explore the transfer robustness between PC-Nets and CNNs on MNIST dataset. Data conversion details between point clouds and MNIST images are provided as follows: 1) \textbf{From images to point clouds}, we convert each pixel to its normalized coordinate $(x,y)$ and the gray-scale value to $z$-axis, which results in point cloud $\{(x,y,z)\}$. 2) \textbf{From (attacked) point clouds to images}, we firstly abandon the points whose $x$ / $y$ coordinates are less than 0 or more than 1, which are out of border; we then convert $(x,y)$ to the corresponding pixel position in image with $\frac{1}{28}$ quantification, and $z$ to pixel gray-scale value. If several points are converted to an overlapping pixel position in image, we adapt their average $z$ values. For the pixels without corresponding points, we set their values as zero. It is worth noting that all the pixels (instead of non-zero pixels only) in the images are converted into point clouds, since there are too many non-zero pixels on the attacked images using CNNs.

To fairly compare the transferability performance, a LeNet-like CNN \cite{lecun1998gradient} is trained on the MNIST image dataset, with an accuracy of 99.1\%; a PointNet \cite{qi2016pointnet} is trained on the MNIST-converted point clouds, with a comparable accuracy (99.2\%). Our experiments indicate that adversarial \tifs{examples} generated by PC-Net are partially transferable to CNN, however those generated by CNN are hardly transferable to PC-Net.

\paragraph{Transferring from PC-Net to CNN}
We apply Pointwise Gradient method to attack the PC-Net (PointNet) which classifies MNIST. The post-attacking accuracy $\eta$ is $21\%$. After converting the adversarial point clouds to images, we transfer them to the CNN model, the accuracy of the CNN model drops from $99.1\%$ to $60.8\%$, which implies these adversarial \tifs{examples} are partially transferable. However, we emphasize that the transferred images are not as ``natural'' as the natural MNIST images. Figure \ref{fig:cnn2point} (Left) demonstrates converted image \tifs{examples} generated by PC-Nets.

\paragraph{Transferring from CNN to PC-Net}
Adversarial \tifs{examples} generated by MNIST CNN under several attack intensities $\mathcal{D}_C$ are transferred to PC-Net. As illustrated in Figure \ref{fig:cnn2point} (Right), the PointNet shows notable robustness: it achieves $98.9\%$ accuracy in the worst case, where the CNN incorrectly classifies all the adversarial \tifs{examples} ($0\%$ accuracy). 

\begin{figure}
	\centering
	\includegraphics[width=\linewidth]{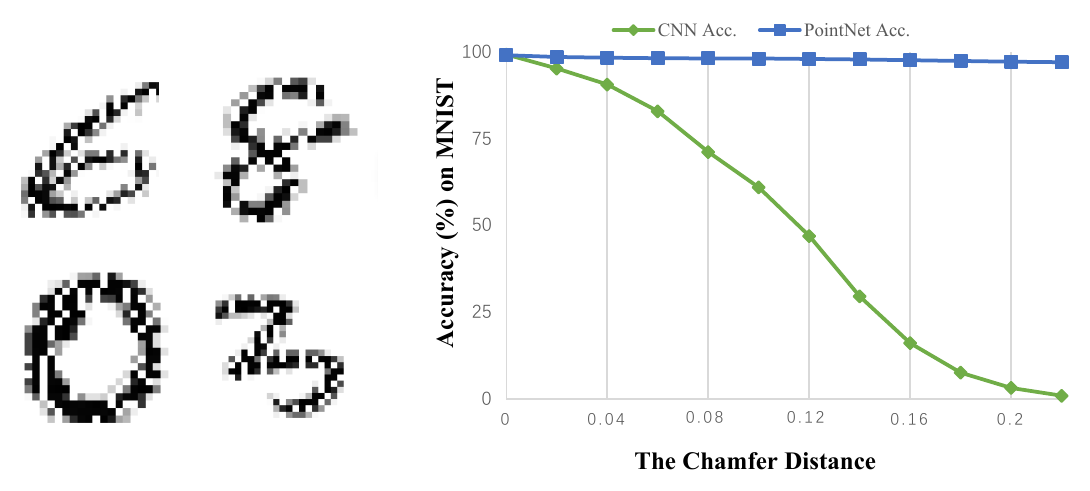} 
	\caption{\textbf{Left}: The converted 2D images from MNIST \tifs{adversarial examples} attacked by PC-Nets.
		\textbf{Right}: Transferred attack performance on PC-Net by adversarial \tifs{examples} generated by CNN. The victim CNN is vulnerable to the adversarial \tifs{examples}, however these adversarial \tifs{examples} are hardly transferable to PC-Net.}
	\label{fig:cnn2point}
\end{figure}


\section{Discussion}
\label{Discussion}

\paragraph{Robustness on Rotation}
Rotation does not produce adversarial \tifs{examples} to PointNet effectively, since the PointNet is sensitive to {\bf random} rotation. 

In our experiment, rotation data augmentation has been applied during training the PointNet with T-Net \cite{qi2016pointnet}, which is mimicking spatial transformers \cite{jaderberg2015spatial}. Although the authors declare robustness on PointNet equipped with T-Net, only 0.3-rad random rotation drops the accuracy from $89.2\%$ to $77.53\%$. We attempt to attack the victim PointNet within same degree constraint (0.3 rad) by adversarial rotation, the rotation matrix parameterized by Euler angles is attacked by back-propagating the gradient, see Appendix \ref{appedix:rotation-attack} for details of the rotation attack. Surprisingly, the post-attacking accuracy is $76.01\%$, very closed to the random rotation accuracy $77.53\%$. We conjecture that the weak robustness on rotation provides non-informative gradients (``stochastic gradients" \cite{athalye2018obfuscated}) to the gradient based attack. Though the adversarial rotation attack is not effective enough compared to random setting, the PointNet is not robust to rotation. For AI security, PC-Nets should also be designed with theoretically guaranteed rotation-invariance \cite{cohen2016group, cohen2018spherical}. We leave this for further exploration.

\paragraph{Two Types of Defense Errors}

We characterize two types of defense errors with {\em perturbation-measurement} framework.
{\bf Defense Error I} is the error caused by the defense framework. As illustrated in Figure \ref{fig:defense} (Left), there are overlaps between the distributions of adversarial \tifs{examples} and natural \tifs{examples}. To reduce the Defense Error I, advanced perturbations and measurements with fine-tuned hyper-parameters should be explored.

As we explore the intrinsic robustness of the victim model, we suffer from its weakness by design, resulting in {\bf Defense Error II}. Refer to Figure \ref{fig:defense} (Right), the distribution of $Correct$ and $All$ natural \tifs{examples} is sightly different.
We observe that the defense performance compared with $Correct$ natural \tifs{examples} is better than that with $All$ natural \tifs{examples}, which means the misclassified natural \tifs{examples} by PC-Nets are more likely to be mistakenly regarded as adversarial \tifs{examples}. We choose to compare with the $Correct$ \tifs{examples} in defense methods to less count the Defense Error II. Theoretically, a better-performing model is able to lower the Defense Error II.


\paragraph{Limitations and Further Work}

There are several limitations of our study. Although we have conducted experiments on black-box attack and transferability, all attack methods are developed in the white-box setting. More research on the black-box attack and the defense techniques in a black-box setting is needed. Besides, our Point-Detach attack method is designed for PointNet, by utilizing the critical-point property. Though the critical-point property lays in various PC-Nets, a universal Point-Detach method should be developed. It is also a future direction.

\tifs{
Moreover, it is also very interesting to explore the physical adversarial point clouds with 3D printing, even though it is a challenging problem. There are several issues for the physical adversarial point clouds: how to deal with the multiple data formats used in the 3D sensors (point clouds, RGB-D images, meshes or voxels), and how to deal with the illumination and sensing noises. These issues make the physical adversarial 3D object generation much more challenging than 2D images.

Although we propose a effective defense strategy to detect the adversarial point clouds, which is robust even with a proof-of-concept attack targeting it, it is expected that any defense strategy is not strong enough in a long term. Considering the increasing importance of reasoning on adversarial point clouds, more research on the adversarial safety is demanding. 
}

\section{Conclusion}
Deep learning on 3D point clouds is an emerging field. In our study, we formulate and discuss the adversarial \textbf{attacks}, \textbf{defenses} and \textbf{transferability} of adversarial clouds. 

We propose three effective attack methods on point clouds. Pointwise Gradient (PG) methods are generally effective on various point cloud networks with high success rate. Point-Attach (PA) and Point-Detach (PD) methods have lower success rate, but are much more physically feasible. 
Besides, a permutation-measurement defense framework is proposed with empirical results. \tifs{Notably, our defense method is still effective to detect the adversarial point clouds generated by a proof-of-concept attack directly targeting it}. 
Furthermore, we investigate the transferability of adversarial \tifs{examples} under 2 setting, 1) between various point cloud networks, and 2) between point clouds and grid CNNs. Momentum-enhanced attack are remarkably effective to improve the transferability of the adversarial point clouds. Grid CNNs are observed less robust than the point cloud approach in our experiments. 

In further study, we are interested in more exploration on physically feasible attack on the 3D objects, black-box attack, and defense strategies on boosting the robustness of the point cloud networks.


%






%

\bibliography{ref}
\bibliographystyle{ieee}

%

\begin{IEEEbiography}[{\includegraphics[width=1in,height=1.25in,clip,keepaspectratio]{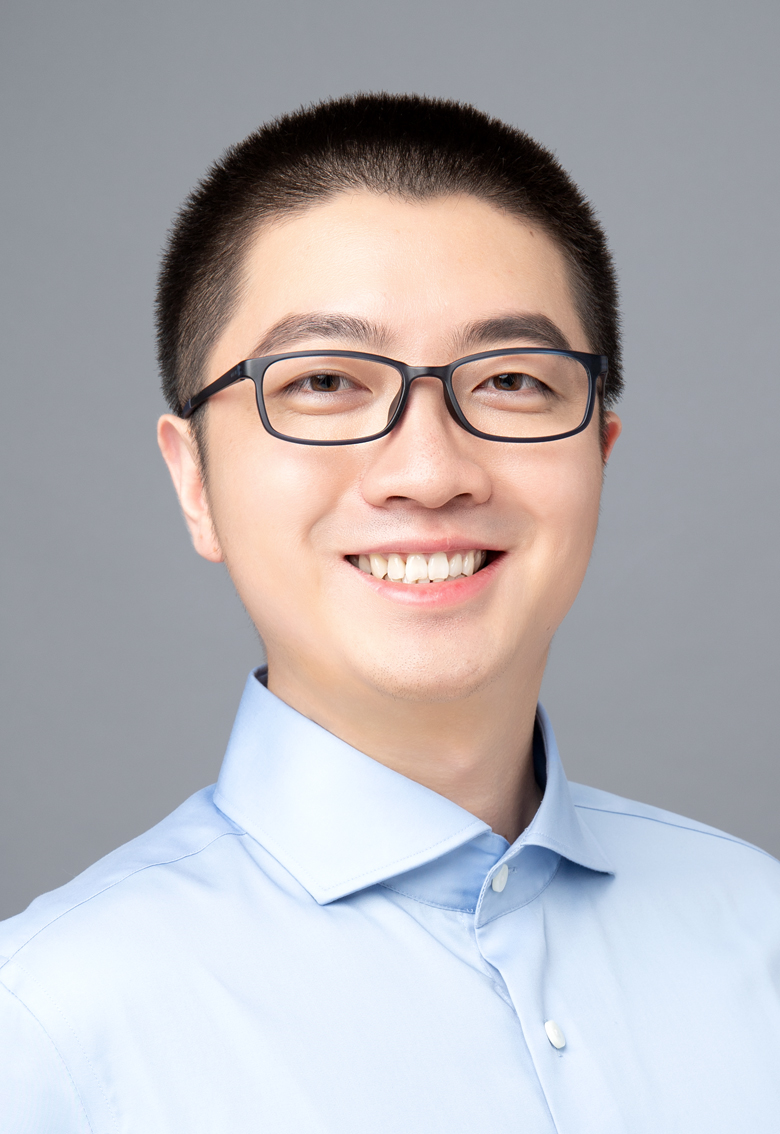}}]{Jiancheng Yang}
	
	received the B.Eng. and M.Eng. degree in Automation from Shanghai Jiao Tong University, China, in 2015 and 2018, respectively. He also received engineer degree (double master degree) in Institut Mines-Telecom, France, in 2016. He is working towards the Ph.D. degree with Shanghai Jiao Tong University. His research interests are deep learning, with emphasis on medical image analysis, 3D computer vision and trustworthy machine learning.
\end{IEEEbiography}

\begin{IEEEbiography}[{\includegraphics[width=1in,height=1.25in,clip,keepaspectratio]{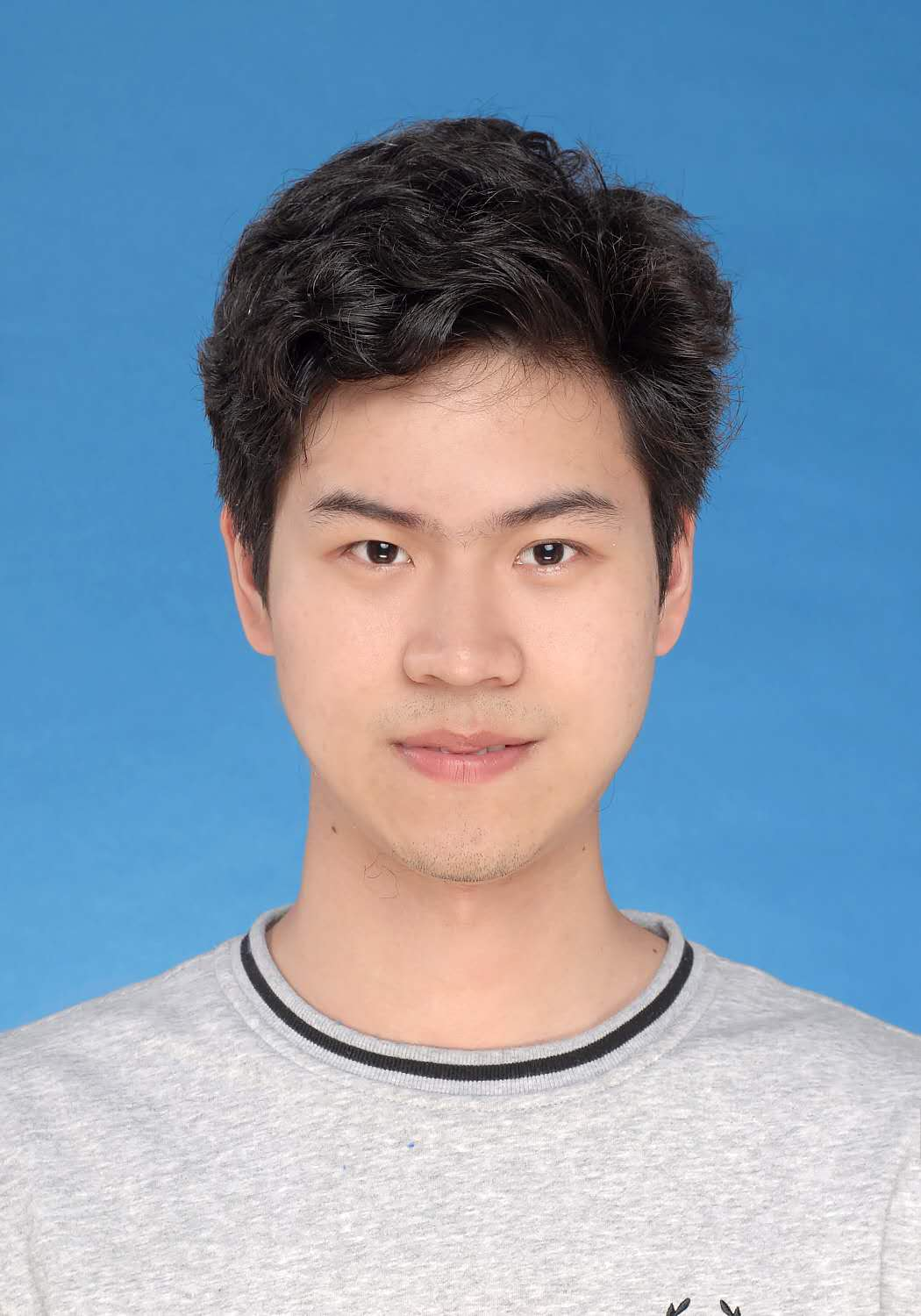}}]{Qiang Zhang}
	is now a B.Eng. student in Computer Science and Engineering at Shanghai Jiao Tong University. He won national championship in Robocup China Open 2018. His research interests are artificial intelligence, computer vision and machine learning, especially in 3D vision, robot vision and autonomous driving.
\end{IEEEbiography}

\begin{IEEEbiography}[{\includegraphics[width=1in,height=1.25in,clip,keepaspectratio]{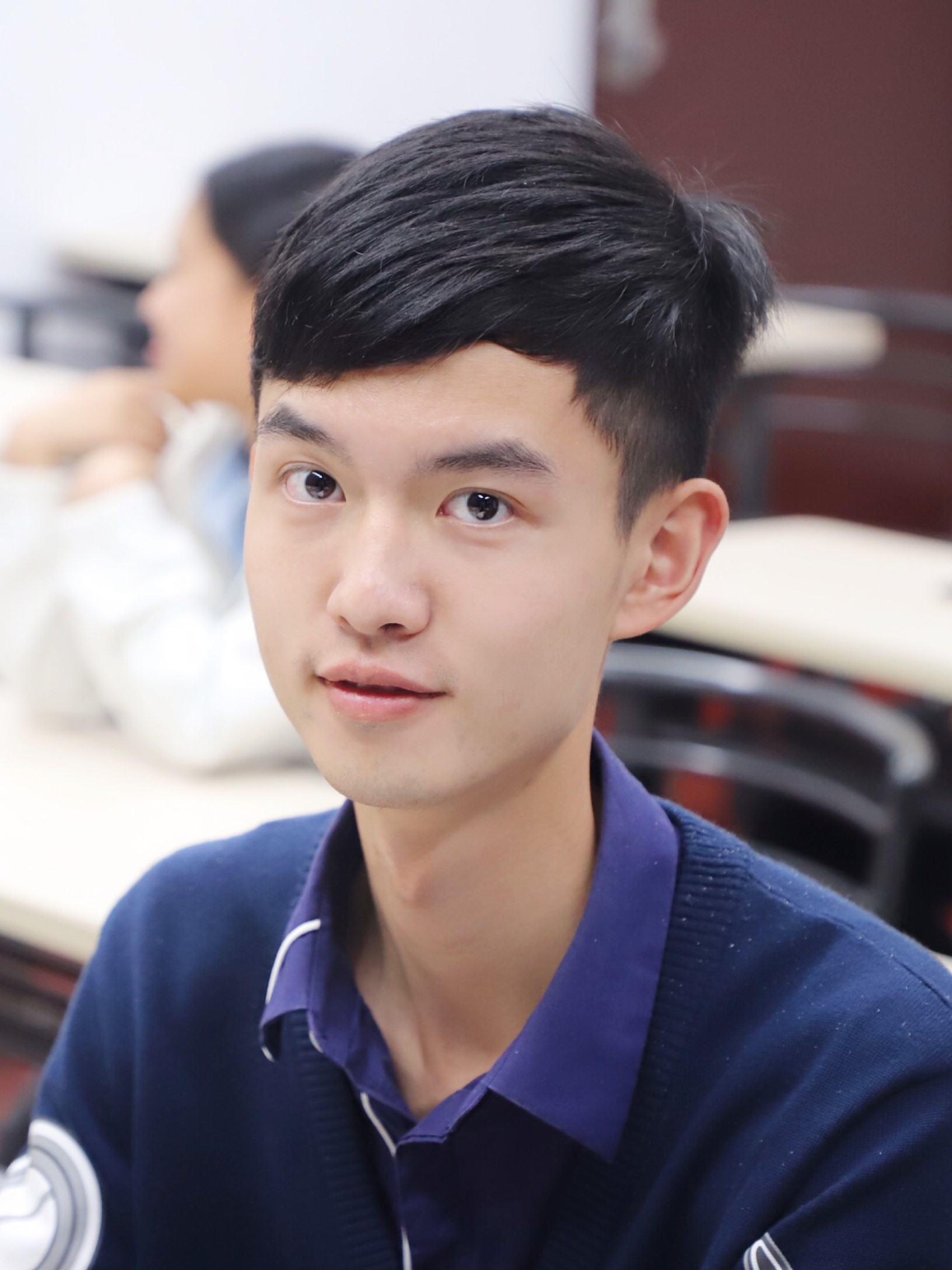}}]{Rongyao Fang}
	is now a B.Eng. student in Electronic Engineering at Shanghai Jiao Tong University. His research interests lay in computer vision and deep learning, particularly 3D computer vision and medical imaging, as well as the application in adversarial example.
\end{IEEEbiography}

\begin{IEEEbiography}[{\includegraphics[width=1in,height=1.25in,clip,keepaspectratio]{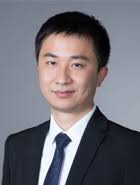}}]{Bingbing Ni}
	received the B.Eng. degree in Electrical Engineering from Shanghai Jiao Tong University, Shanghai, China, in 2005, and the Ph.D. degree from the National University of Singapore, Singapore, in 2011. He is currently a Professor at the Department of Electrical Engineering, Shanghai Jiao Tong University. Before that, he was a Research Scientist in Advanced Digital Sciences Center, Singapore. He was with Microsoft Research Asia, Beijing, China, as a Research Intern in 2009. He was also a Software Engineer Intern with Google Inc., Mountain View, CA, USA, in 2010. Dr. Ni was a recipient of the Best Paper Award from PCM’11 and the Best Student Paper Award from PREMIA’08. He was also the recipient of the first prize in International Contest on Human Activity Recognition and Localization in conjunction with International Conference on Pattern Recognition in 2012.
\end{IEEEbiography}

\begin{IEEEbiography}[{\includegraphics[width=1in,height=1.25in,clip,keepaspectratio]{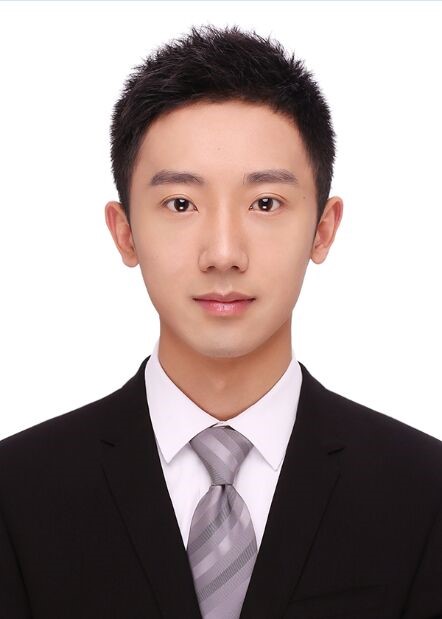}}]{Jinxian Liu}
	received the B.Eng. degree in communication engineering from the Tianjin University, in 2017. He is working towards the Ph.D. degree with Shanghai Jiao Tong University. His research interests are primarily on machine learning and computer vision with applications to 3D data analysis and set learning.
\end{IEEEbiography}

\begin{IEEEbiography}[{\includegraphics[width=1in,height=1.25in,clip,keepaspectratio]{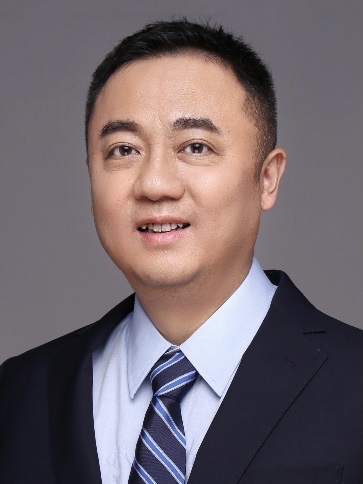}}]{Qi Tian}
	(S’95-M’96-SM’03-F’16) is currently a Chief Scientist of Computer Vision in Huawei Noah's Ark Laboratory, a Full Professor with the Department of Computer Science, University of Texas at San Antonio (UTSA), USA. He was a tenured Associate Professor from 2008-2012 and a tenure-track Assistant Professor from 2002-2008. During 2008-2009, he took one-year Faculty Leave at Microsoft Research Asia (MSRA) as Lead Researcher in the Media Computing Group.
	
	Dr. Tian received his Ph.D. in ECE from University of Illinois at Urbana-Champaign (UIUC) in 2002 and received his B.E. in Electronic Engineering from Tsinghua University in 1992 and M.S. in ECE from Drexel University in 1996. Dr. Tian’s research interests include multimedia information retrieval, computer vision, pattern recognition and published over 360 refereed journal and conference papers. He was the co-author of a Best Paper in ACM ICMR 2015, a Best Paper in PCM 2013, a Best Paper in MMM 2013, a Best Paper in ACM ICIMCS 2012, a Top 10\% Paper Award in MMSP 2011, a Best Student Paper in ICASSP 2006, and co-author of a Best Student Paper Candidate in ICME 2015, and a Best Paper Candidate in PCM 2007. 
	
	Dr. Tian received 2017 UTSA President’s Distinguished Award for Research Achievement, 2016 UTSA Innovation Award, 2014 Research Achievement Awards from College of Science, UTSA, 2010 Google Faculty Award, and 2010 ACM Service Award. He is the associate editor of many journals and in the Editorial Board of Journal of Multimedia (JMM) and Journal of Machine Vision and Applications (MVA).  
\end{IEEEbiography}





\clearpage
\appendices
\setcounter{table}{0}
\renewcommand{\thetable}{A\arabic{table}}
\setcounter{figure}{0}
\renewcommand{\thefigure}{A\arabic{figure}}
\setcounter{proposition}{0} 
\setcounter{assumption}{0} 
\newtheorem*{lemma}{Lemma}

\section{Proofs} \label{sec:proposition-proof}
\subsection{On Pointwise-Gradient Attack}
\begin{proposition}
	Given any point cloud dataset $\mathbb{X}$, $\exists{\epsilon} \in \mathbb{R}$, $\forall{X \in \mathbb{X}}$ s.t. $\arg\max f(X)=c^{*}(X)$, $\exists{X_a}=T_{PG}(X): $ $\mathcal{D}_C(X,X_a)<\epsilon$ and $\arg\max f(X_a)\neq c^{*}(X) $.
\end{proposition}

\begin{proof}
	
	Ultimately, for any sample $X$ with label $t$, we could pointwisely change it into another sample $X'$ with label $t'$. 
	
	In this way, $\epsilon_X=\mathcal{D}_C(X,X')$. For all samples in the dataset, $\epsilon=\max_{X \in \mathbb{X}}(\epsilon_X)$.
	
\end{proof}

\subsection{On Point-Attach Attack}
\begin{proposition}
	Given any point cloud dataset $\mathbb{X}$, $\exists{\epsilon} \in \mathbb{R}$, $\exists{N_a} \in \mathbb{N}$, $\exists{X \in \mathbb{X}}$ s.t. $\arg\max f(X)=c^{*}(X)$, $\exists{X_a}=T_{PA}(X): $ $\mathcal{D}_C(X,X_a)<\epsilon$, $\Delta N_{X,X_a} \leq N_a$, and $\arg\max f(X_a)\neq c^{*}(X) $.
\end{proposition}
\begin{proof}
	By contradiction. 
	
	Assume there is not such sample in the dataset. We choose two \tifs{examples} $X_1$ and $X_2$ with different labels, then attach one to another as a new sample $X'$. If the prediction label is the same as the $X_1$, then we successfully attack $X_2$ by attaching $X_1$. Otherwise, we successfully attack $X_1$ by attaching $X_2$.
\end{proof}

\subsection{On Max Confidence-Based Measurement}
\begin{assumption}
	PC-Nets are local continuous convex or concave functions around natural \tifs{examples}. 
\end{assumption}

\begin{assumption} The proportion of adversarial \tifs{examples} in the local area around natural \tifs{examples} is small enough, i.e., given \tifs{benign sample} $X_n$, 
	\begin{align*}
	&\exists{\delta}>0, \text{define} \, \mathbb{D}_{\delta}=\{X \,|\,\mathcal{D}_{C}(X,X_n)<\delta\}, \\&\mathbb{B}_\delta= \{X_a \, is \, adversarial \,\tifs{example} \,|\, \mathcal{D}_C(X_a,X_n)<\delta\} \subset \mathbb{D}_\delta,
	\\&\exists{\epsilon} \ll 1: \forall{X}, P(X\in\mathbb{B}_\delta)/P(X\in\mathbb{D}_\delta)<\epsilon .
	\end{align*}
\end{assumption}

\begin{proposition}
	Given any \tifs{sample} $X$, it can always be detected whether it is an adversarial \tifs{example} or a \tifs{benign sample} by Max Confidence-Based Measurement (Section \ref{defense:confidence}) ($CoV$ for Convex functions and $CoA$ for Concave functions), with Gaussian Noising or Quantification perturbation (Section \ref{defense:perturbation}).
\end{proposition}
\begin{proof}
	We prove this proposition in two separate propositions (Proposition \ref{prop-concave} and Proposition \ref{prop-convex}).
\end{proof}

\begin{proposition}[Concave] \label{prop-concave}
	$CoA$ is an effective index to detect whether an given sample X is an adversarial \tifs{example} or a \tifs{benign sample} if PC-Nets are local continuous concave functions.
\end{proposition}

\begin{proof}
	Without losing generality, we choose a dataset containing with only 2 labels: negative and positive, the network output value is limited between 0 and 1. Specifically, if it is more than 0.5, it denotes positive class, otherwise, it denotes negative class. Similar to Gaussian Noising, Quantization is equavalent to add perturbation using uniform distribution. For simplicity, we take uniform distribution as example.
	
	Assume $x_1$ is the quantized data for a negative adversarial \tifs{examples} $X_a$ and $x_2$ is the quantized data of the corresponding positive natural \tifs{examples} $X$. Besides, we apply a same perturbation to $x_1$ and $x_2$, $\mathcal{D}_C(x_1,X_a) \in [a_1,b_1]$, $\mathcal{D}_C(x_2,X) \in [a_2,b_2]$, note $b_1-a_1=b_2-a_2$ (for the ``same perturbation"). Thus, according to the definition of $CoA$, we have:
	$$CoA(x_1)=\frac{1}{b_1-a_1}\int_{a_1}^{b_1}[1-f(x)]dx,$$
	$$CoA(x_2)=\frac{1}{b_2-a_2}\int_{a_2}^{b_2}f(x)dx.$$
	
	Define $f$ is the PC-Net, $f'$ is its derivative. Since $f$ is continuous, $\exists{c}: f(c)=0.5$ when the point set varies from $x_2 (f(x_2)>0.5)$ to $x_1 (f(x_1)<0.5)$. Besides, since $f$ is a concave function via the distance, we have the following equations:
	$$f(x_1)>0.5-|f'(c)|*\mathcal{D}_C(x_1,c),$$
	$$f(x_2)>0.5+|f'(c)|*\mathcal{D}_C(c,x_2).$$
	
	\begin{figure}[!htb]
		\centering
		\includegraphics[width=8cm]{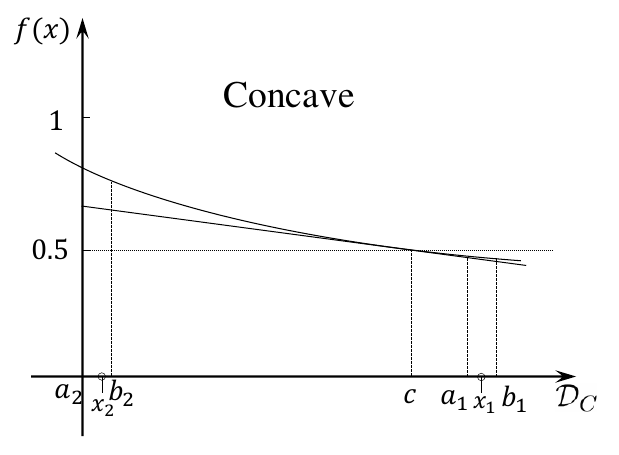}
		\caption{Illustration for Proposition \ref{prop-concave}. The tangent at $f(c)=0.5$ is shown in the figure.}
		\label{Figure:Concave}
	\end{figure}
	
	According to our assumption, adversarial \tifs{examples} are much less than natural \tifs{examples}, as illustrated in Figure \ref{Figure:Concave}:
	$$\mathcal{D}_C(c,x_2)>\mathcal{D}_C(x_1,c).$$
	Thereby, 
	$$f(x_2)>1-f(x_1)$$
	$$\Rightarrow \frac{1}{b_2-a_2}\int_{a_2}^{b_2}f(x)dx >\frac{1}{b_1-a_1}\int_{a_1}^{b_1}[1-f(x)]dx,$$
	which means $CoA(x_2)>CoA(x_1)$, thus $CoA$ is effective to detect adversarial \tifs{examples}.
\end{proof}

To prove the next proposition, we provide following lemma:
\begin{lemma}
	$\forall{x}$, $\forall{f(x)}$, $\forall{c}$:
	$\int_{a}^b[f(x)-\frac{1}{b-a}\int_a^bf(x)dx]^2dx \leqslant \int_a^b(f(x)-c)^2dx $
\end{lemma}

\begin{proof} 
	\begin{align*}
	&\int_a^b[f(x)-\frac{1}{b-a}\int_a^bf(x)dx]^2dx \\
	&=\int_a^b(f(x)-c)^2dx+2(c-\frac{1}{b-a}\int_a^bf(x)dx)\times \\
	&\ \ \ \int_a^b(f(x)-c)dx+(c-\frac{1}{b-a}\int_a^bf(x)dx)^2(b-a) \\
	&=\int_a^b(f(x)-c)^2dx+\frac{2}{a-b}[\int_a^b(f(x)-c)dx]^2 \\
	&\ \ \ +\frac{1}{b-a}[\int_a^b(f(x)-c)dx]^2\\
	&\leqslant \int_a^b(f(x)-c)^2dx
	\end{align*}
\end{proof}

\begin{proposition}[Convex] \label{prop-convex}
	$CoV$ is an effective index to detect whether an given sample X is an adversarial \tifs{example} or a \tifs{benign sample} if PC-Nets are local continuous convex functions.
\end{proposition}

\begin{figure}[!htb]
	\centering
	\includegraphics[width=8cm]{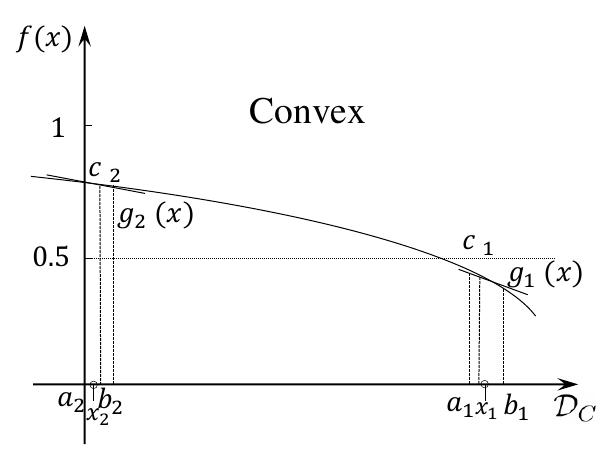}
	\caption{Illustration for the Proposition \ref{prop-convex}. The two straight lines are parameterized by $g_1(x)=f'(a_1)x+l_1, s.t. \, g_1(c_1)=f(c_1)$, $g_2(x)=f'(b_2)x+l_2,s.t.\,f(c_2)=g(c_2)$.}
	\label{Figure:Convex}
\end{figure}

\begin{proof}
	We use the same notations in the Proposition \ref{prop-convex}.
	\item{For adversarial \tifs{example}:}
	according to mean value theorem of integrals, 
	$\exists c_1, f(c_1) = \frac{1}{b_1-a_1}\int_{a_1}^{b_1}f(x)dx$, we construct:
	$g_1(x)=f'(a_1)x+l_1, s.t. g_1(c_1)=f(c_1)$.
	
	\begin{align*}
	CoV(x_1) &= \int_{a_1}^{b_1}[f(x)-f(c_1)]^2dx\\
	&>\int_{a_1}^{b_1}[g(x)-f(c_1)]^2dx \,\text{(illustrated in Figure \ref{Figure:Convex})}\\  
	&=\int_{a_1}^{b_1}[g(x)-g(c_1)]^2dx\\
	&\ge\int_{a_1}^{b_1}[g(x)-\frac{1}{b_1-a_1}\int_{a_1}^{b_1}g(x)dx]^2dx \,\text{(Lemma)}\\
	&=\int_{a_1}^{b_1}[f'(a_1)x-\frac{f'(a_1)(a_1+b_1)}{2}]^2dx\\
	&=\frac{1}{12}f'(a_1)^2(b_1-a_1)^3
	\end{align*}
	
	\item{For \tifs{benign sample}:}
	according to our assumption that $f(x)$ is continuous, we can choose $c_2$ satisfying $D_C(c_2,x_2)=\frac{a_2+b_2}{2}$ and construct $g_2(x)=f'(b_2)x+l_2,s.t.f(c_2)=g(c_2)$.
	\begin{align*}
	CoV(x_2)&=\int_{a_2}^{b_2}[f(x)-\frac{1}{{b_2}-{a_2}}\int_{b_2}^{b_2}f(x)dx]^2dx\\
	&\leq \int_{a_2}^{b_2}[f(x)-f(c_2)]^2dx \,\text{(Lemma)}\\
	&<\int_{a_2}^{b_2}[g_2(x)-f(c_2)]^2dx \,\text{(illustrated in Figure \ref{Figure:Convex})}\\
	&=\int_{a_2}^{b_2}[g_2(x)-g(c_2)]^2dx\\
	&=\int_{a_2}^{b_2}[f'(b_2)x-\frac{f'(b_2)(a_2+b_2)}{2}]^2dx\\
	&=\frac{1}{12}f'(b_2)^2(b_2-a_2)^3\\
	\end{align*}
	
	Since $a_1>b_2$ and $f(x)$ is a convex function, $f'(a_1)^2>f'(b_2)^2$. Besides, we have $b_2-a_2=b_1-a_1$, which means $CoV(x_1)>CoV(x_2)$. Thereby, $CoV$ is effective to detect adversarial \tifs{examples} and \tifs{benign samples}.
\end{proof}

\begin{table*}
		
	\caption{More defense results by applying the perturbation and measurement methods against various attack.}
		
	\begin{center}
		\begin{tabu} to 1\textwidth{X[c]X[c]X[c]X[c]X[c]X[c]X[c]X[c]X[c]X[c]X[c]}
			\toprule[1pt]
			\multicolumn{2}{c}{\multirow{2}{*}{AUROC}} & \multicolumn{3}{c}{Random Sampling}       & \multicolumn{3}{c}{Quantification} &   \multicolumn{3}{c}{Gaussian Noising}                    \\
			\multicolumn{2}{c}{}  & SIV     & CoA & CoV & SIV        & CoA    & CoV   & SIV        & CoA   & CoV  \\ \hline                    
			\multicolumn{2}{c}{Pointwise Gradient} & \multicolumn{3}{c}{$n = 500$}       & \multicolumn{3}{c}{$\mu = 0.08$} &   \multicolumn{3}{c}{$\sigma = 0.020$}                    \\ \hline
			\multirow{2}{*}{$\mathcal{D}_C = 0.02$}    & All      & 0.8806  & 0.8544    & 0.8749   & \textbf{0.9168}    & 0.8639       & 0.9086      & 0.8839     & 0.8589       & 0.8816     \\
			& Correct  & 0.9132   & 0.9006    & 0.9054    & \textbf{0.9381}     & 0.9125       & 0.9323      & 0.9207    & 0.9040       & 0.9192     \\ 
			\hline
			
			\multicolumn{1}{l}{}  & \multicolumn{1}{l}{} & \multicolumn{3}{c}{$n=500$} & \multicolumn{3}{c}{$\mu = 0.12$} & \multicolumn{3}{c}{$\sigma = 0.028$} \\ \hline
			\multirow{2}{*}{$\mathcal{D}_C = 0.05$} & All                  & 0.8092    & 0.74283  & 0.80705  & \textbf{0.8352}       & 0.75224     & 0.82039     & 0.78812  & 0.74675 & 0.78686 \\
			& Correct              & 0.84692   & 0.79698  & 0.84286  & \textbf{0.85932}      & 0.80872     & 0.84934     & 0.82665  & 0.79681 & 0.82403 \\ \hline

			\multicolumn{2}{c}{Point-Detach}                   & \multicolumn{3}{c}{$n = 950$} & \multicolumn{3}{c}{$\mu=0.02$} & \multicolumn{3}{c}{$\sigma = 0.012$} \\ \hline
			\multirow{2}{*}{$N_d = 20$}         & All     & 0.9053  & 0.9083  & 0.9016   & 0.9154     & 0.9109        & 0.9097      & \textbf{0.9229}    & 0.9139        & 0.9152     \\ 
			& Correct & 0.9393 & 0.9439 & 0.9357   & 0.9454    & 0.9455       & 0.9407      & \textbf{0.9517}    & 0.9475       & 0.9453     \\ 
			\hline
			\multirow{2}{*}{$N_d = 40$} & All     & 0.8705 & 0.8770  & 0.8681 & 0.8901  & 0.8827 & 0.8863 & \textbf{0.9011} & 0.8903 & 0.8969  \\
			& Correct & 0.9130 & 0.9190 & 0.9102 & 0.9249 & 0.9231 & 0.9216 & \textbf{0.9343}  & 0.9282 & 0.9304  \\ \hline
			\multicolumn{2}{c}{Point-Attach}  &\multicolumn{3}{c}{$n = 1000$} &  \multicolumn{3}{c}{$\mu = 0.02$} &  \multicolumn{3}{c}{$\sigma = 0.012$} \\ \hline
			
			\multirow{2}{*}{$N_a = 16$}  & All                  & \textbf{0.9632}   & 0.8275   & 0.9588  & 0.8058      & 0.8084    & 0.8031     & 0.8169  & 0.8113 & 0.8159 \\
			& Correct & \textbf{0.9750}              & 0.8712   & 0.9720   & 0.84656  & 0.8509      & 0.8442     & 0.8578     & 0.8539  & 0.8565  \\ \hline
			
			\multirow{2}{*}{$N_a = 32$}      & All     & \textbf{0.9729} & 0.8507 & 0.9686  & 0.8366  & 0.8329 & 0.8365 & 0.8516 & 0.8379 & 0.8512 \\
			& Correct & \textbf{0.9816} & 0.8934 & 0.9785 & 0.8752 & 0.8752 & 0.8746  & 0.8890 & 0.8800 & 0.8878 \\ 
			
			
			\bottomrule[1pt]
		\end{tabu}

		\label{appendixform:defense}
	\end{center}
\end{table*}

\tifs{
\section{Experiment Details in this study}

\subsection{Training of PointNet for ModelNet40}
\paragraph{Model}
PointNet network topology follows official model and the loss function is cross-entropy loss for our classification task.
Input pre-process technique is the same as that in official model experiment: zero center, maximum L2 distance for each point to the point cloud center is 1m.
We use a standard normal distribution to randomly initialize the network.

\paragraph{Training}
As for the optimizer, it is Adam with learning rate 0.001, and mini-batch number equals 32 and iteration epoch number is 250.
The training data is randomly visited and shuffled between epochs. 

\paragraph{Implementation}
We use Titan Xp GPUs to train the model with PyTorch(version 1.0). For the criterion to determine hyperparameter we choose original official values.

\subsection{Training of PointNet++ for ModelNet40}
\paragraph{Model}
PoineNet++ network topology follows official model and the loss function is cross-entropy loss for our classification task.
Input pre-process technique is the same as that in official model experiment: zero center, maximum l2 distance for each point to the point cloud center is 1m.
The network is randomly initialized with a standard normal distribution.

\paragraph{Training}
As for the optimizer, it is Adam with initial learning rate 0.001, after every 200000 steps learning rate decays with a ratio 0.7, at the same time, the minimum learning rate is 0.00001. Iteration epoch is 250.
The training data is randomly visited and shuffled between epochs. 

\paragraph{Implementation}
We use Titan Xp GPUs to train the model with TensorFlow(version 1.4). For the criterion to determine hyperparameter we choose original official values.

\subsection{Training of DGCNN for ModelNet40}
\paragraph{Model}
DGCNN network follows official network and the loss function is cross-entropy loss for our classification task.
Input pre-process technique is the same as that in official model experiment: zero center, maximum l2 distance for each point to the point cloud center is 1m.

\paragraph{Training}
As for the optimizer, it is Adam with initial learning rate 0.001, after every 200000 steps learning rate decays with a ratio 0.7, at the same time, the minimum learning rate is 0.00001. Iteration epoch is 250.
The training data is randomly visited and shuffled between epochs. 

\paragraph{Implementation}
We use Titan Xp GPUs to train the model with TensorFlow(version 1.4). For the criterion to determine hyperparameter we choose original official values.

\subsection{Training of PointNet for MNIST}
\paragraph{Model}
PointNet network topology follows official model and the loss function is cross-entropy loss for our classification task.
Input pre-process: for each image, we transfer each pixel into a point, assume the pixel coordinate is (x,y), whose grey value is c, then the point coordinate would be (x,y,c), finally, the obtained point cloud will be processed as that in the modelnet40 dataset.
The network is randomly initialized with a standard normal distribution.

\paragraph{Training}
As for the optimizer, it is Adam with learning rate 0.001 and iteration epoch number is 25.
The training data is randomly visited and shuffled between epochs. 

\paragraph{Implementation}
We use Titan Xp GPUs to train the model with PyTorch(version 1.0). For the criterion to determine hyperparameter We use cross-validation and grid search technique to select relatively approximate epoch number and learning rate.

\subsection{Training of CNN for MNIST}
\paragraph{Model}
Our CNN network is the same as LeNet and the loss function is cross-entropy loss.
The network is randomly initialized with a standard normal distribution.

\paragraph{Training}
As for the optimizer, it is Adam with learning rate 0.001 and iteration epoch number is 25.
The training data is randomly visited and shuffled between epochs. 

\paragraph{Implementation}
We use Titan Xp GPUs to train the model with PyTorch(version 1.0). For the criterion to determine hyperparameter We use cross-validation and grid search technique to select relatively approximate epoch number and learning rate.

}
\section{Details of Rotation Attack} \label{appedix:rotation-attack}

In Section \ref{Discussion}, the robustness on rotation of PointNet \cite{qi2016pointnet} is evaluated. We use Euler angles (Figure \ref{Figure:Rotation}) to parameterize rotation, and present the attack method as follows.

\begin{align*}
M_{\alpha}&=
\left[ \begin{array}{ccc}
\cos{\alpha} & \sin{\alpha} & 0\\
-\sin{\alpha} & \cos{\alpha} & 0\\
0 & 0 & 1
\end{array} \right ] \\
M_{\beta}&=
\left[ \begin{array}{ccc}
1 & 0 & 0\\
0& \cos{\beta} & \sin{\beta}\\
0 & -\sin{\beta} & \cos{\beta}
\end{array} \right ] \\
M_{\gamma}&=
\left[ \begin{array}{ccc}
\cos{\gamma} & \sin{\gamma} & 0\\
-\sin{\gamma}& \cos{\gamma} & 0\\
0 & 0 & 1
\end{array} \right ] \\
X_a &= X \ M_{\gamma}\  M_{\beta}\ M_{\alpha}\\
\nabla \theta&= \frac{\partial{f^{(t)}(X_a)}}{\partial{\theta}}, \theta= (\alpha, \beta, \gamma)
\end{align*}
Following a gradient-based iteration (e.g., iteration in Pointwise Gradient method), we update the $\theta$ adversarially. Each dimension of $\theta$ (i.e., $\alpha$, $\beta$, and $\gamma$) is limited in 0.3 rad.

\begin{figure}[h]
	\centering
	\includegraphics[width=8cm]{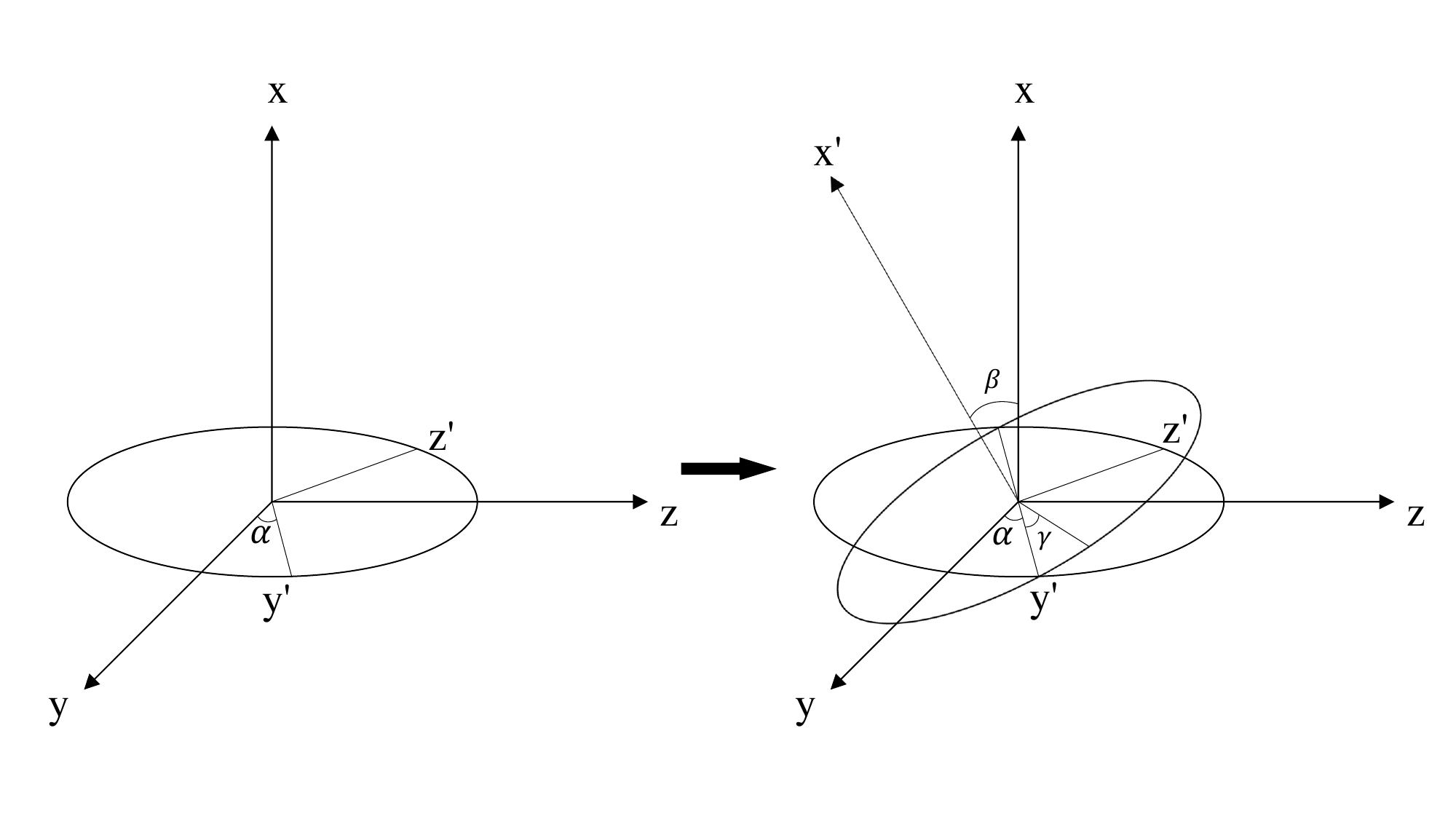}
	\caption{Rotation system using Euler angles $\alpha$, $\beta$, and $\gamma$. The rotation axes are $x$-axis (left), $y'$-axis (right) and $x'$-axis (right).}
	\label{Figure:Rotation}
\end{figure}

\section{More Defense Results}
\label{appendix:defense}
\label{sec:more-defense}
More defense results are depicted in Table \ref{appendixform:defense}. It provides defense AUROC with all combinations of the proposed Perturbation (Gaussian Noising, Quantification, and Random Sampling) $\times$ Measurement ($SIV$, $CoA$, and $Cov$) methods. 

\tifs{
\section{Visualization on Pointwise-Gradient Attack} 
\label{appendix-vis}
To illustrate the distortion of point cloud under Pointwise Gradient attack with various intensities, we visualize several point cloud \tifs{examples} under different attack levels in Figure \ref{fig:attack_vs_chamfer}. It is observed that certain point cloud \tifs{examples} become visibly distorted with the Chamfer distance larger than 0.02.

\begin{figure*}
	\centering
	\includegraphics[width=15.9cm]{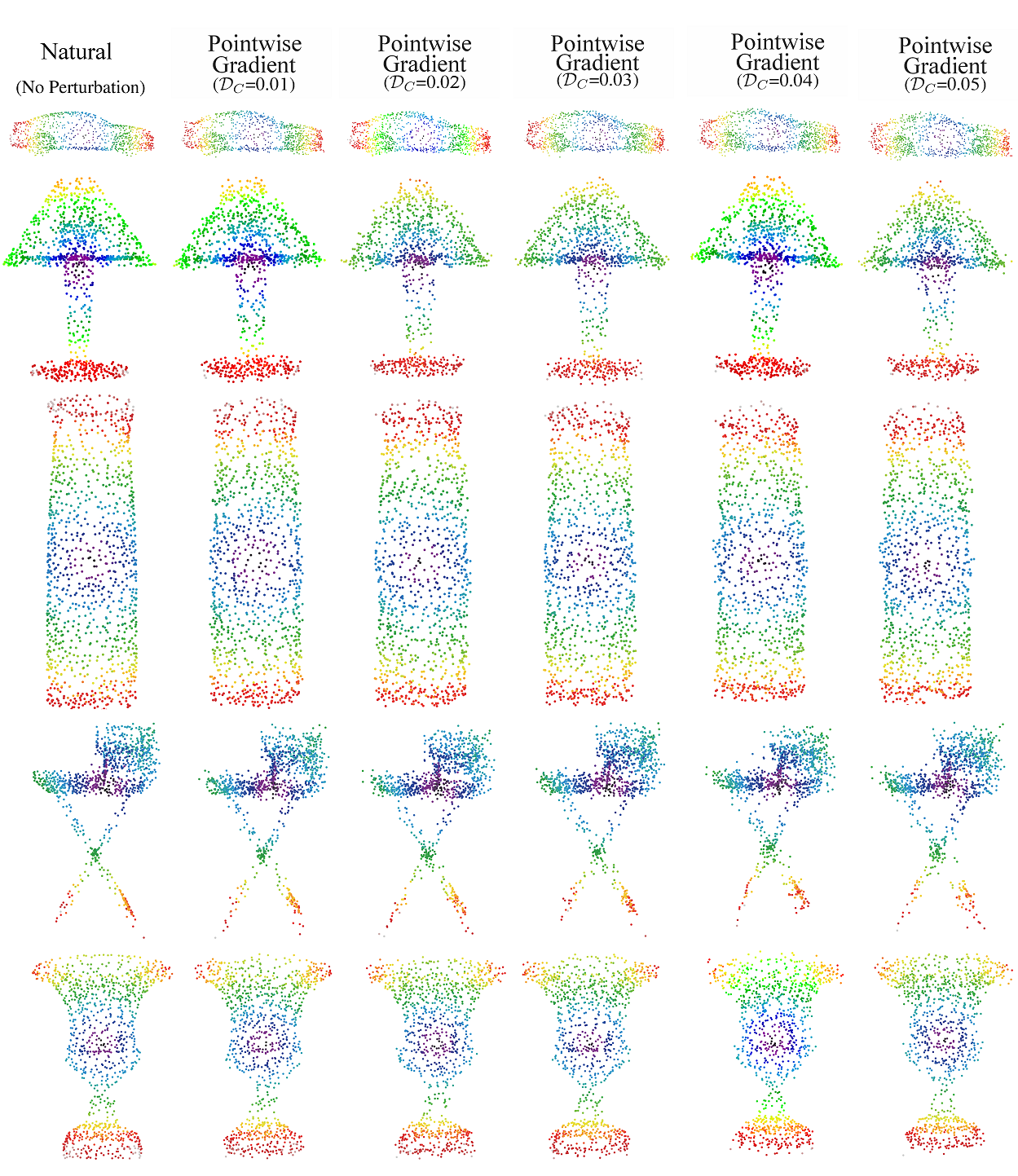}
	
	\caption{\tifs{Benign samples} and Pointwise Gradient attacked \tifs{adversarial examples} with Chamfer distances varying from 0.01 to 0.05. Certain \tifs{examples} become visibly distorted with the Chamfer distance larger than 0.02.}
	\label{fig:attack_vs_chamfer}
\end{figure*}
}
\end{document}